%% file: arxiv_version.tex
\documentclass{article}

\usepackage{arxiv}

\usepackage[utf8]{inputenc} 
\usepackage[T1]{fontenc}    
\usepackage{hyperref}       
\usepackage{url}            
\usepackage{booktabs}       
\usepackage{amsfonts}       
\usepackage{nicefrac}       
\usepackage{microtype}      
\usepackage{lipsum}		
\usepackage{graphicx}
\usepackage[style=plain]{floatrow}
\usepackage{subcaption}
\usepackage[round]{natbib}

\usepackage{mathtools}

\DeclareMathOperator*{\rank}{rank}

\DeclareMathOperator*{\MSE}{MSE}
\DeclareMathOperator*{\RMSE}{RMSE}
\DeclareMathOperator*{\Beta}{Beta}
\DeclareMathOperator*{\Unif}{Unif}

\def \E {\mathbb{E}}

\DeclarePairedDelimiter\norm{\lVert}{\rVert}

\usepackage{algorithmic}
\usepackage{algorithm2e}
\RestyleAlgo{ruled}

\usepackage{amssymb}
\usepackage{amsmath}
\usepackage{amsthm}

\usepackage{xcolor}
\definecolor{DarkGreen}{rgb}{0.1,0.5,0.1}
\usepackage{rotating}

\DeclareMathOperator*{\argmin}{arg\,min}
\usepackage{thmtools}
\usepackage{thm-restate}
\usepackage[nobiblatex]{xurl}
\usepackage{dblfloatfix}
\usepackage{wrapfig}

\newtheorem{definition}{Definition}[section]
\newtheorem{theorem}{Theorem}[section]
\newtheorem{lemma}[theorem]{Lemma}

\newtheorem{corollary}[theorem]{Corollary}

\usepackage{doi}

\title{ClusterSC: Advancing Synthetic Control with Donor Selection
}

\author{ Saeyoung Rho\\ 
	Columbia University
    \And
	Andrew Tang\\ 
	Columbia University
 \And
	Noah Bergam\\ 
	Columbia University
 \And
	Rachel Cummings\\ 
	Columbia University
 \And
	Vishal Misra\\ 
	Columbia University
}

\hypersetup{
pdftitle={Cluster Synthetic Control}
}

\begin{document}
\maketitle

\begin{abstract}
    In causal inference with observational studies, synthetic control (SC) has emerged as a prominent tool.
    SC has traditionally been applied to aggregate-level datasets, but more recent work has extended its use to individual-level data.
    As they contain a greater number of observed units, this shift introduces the curse of dimensionality to SC.
    To address this, we propose Cluster Synthetic Control (ClusterSC), based on the idea that groups of individuals may exist where behavior aligns internally but diverges between groups.
    ClusterSC incorporates a clustering step to select only the relevant donors for the target.
    We provide theoretical guarantees on the improvements induced by ClusterSC, supported by empirical demonstrations on synthetic and real-world datasets.
    The results indicate that ClusterSC consistently outperforms classical SC approaches.
\end{abstract}

\section{Introduction}
\label{sec.introduction}
Synthetic control (SC) has emerged in the econometrics community as a natural extension of the Difference-in-Differences method (D-in-D, \cite{card1993minimum}).
By leveraging time-series data from both pre- and post-intervention periods, SC evaluates the impact of an intervention on a \emph{target unit} by constructing a synthetic counterfactual using a weighted combination of \emph{donor units}, rather than selecting the nearest neighbor as in D-in-D.
Much of the practical usage of SC has been with aggregate-level data, such as assessing the economic impact of government policies or political events at the state or regional level \citep{abadie2003economic, abadie2010synthetic, abadie2015comparative, kreif2016examination}.

Recently, there has been increasing attention to employing SC on disaggregate-level data, observed in contexts like clinical trials with individual health records \citep{thorlund2020synthetic} and economic analyses using individual income data \citep{abadie2021penalized}.
In disaggregate-level datasets, the number of observed donor units can increase dramatically, easily exceeding the number of time-series measurements.
Although more data typically means more information, the dimension of the synthetic control weights is determined by the number of units in the donor data. Hence, increased number of donors may introduce the \emph{curse of dimensionality}, where learning happens in a high-dimensional space with only a few time-series measurements.

In light of this, we revisit the core motivation of synthetic control, which is to construct a \emph{similar} counterpart to the target unit.
What if there is a group of donors that behaves most similarly to the given target unit?
We hypothesize an underlying group-based structure where the latent variables have a certain \emph{structural separation}.
Specifically, we focus on the distribution of the right singular vectors in each unit, and suggest clustering the donor pool before learning SC weights.
We then analyze the impact of selecting a subgroup of donors, rather than the entire donor pool, within the SC framework.

Our contribution is twofold.
First, we introduce ClusterSC, a novel approach to disaggregate-level SC to mitigate high noise and dimensionality issues by incorporating a donor clustering step. 
Second, we provide a theoretical analysis of our algorithm's guarantees, based on the structural assumptions in the latent variable space.
We also validate our approach empirically on synthetic and real-world datasets, demonstrating the improved prediction accuracy achieved by our method. 

Section \ref{sec.background} introduces the synthetic control family of methods and defines relevant notation.
We formalize the problem setup and introduce structural assumptions in Section \ref{sec.setup}.
Our main algorithm is introduced in Section \ref{sec.algo}, with theoretical analyses in Section \ref{sec.theory}.
Finally, Section \ref{sec.empirical} empirically evaluates the performance of our approach on synthetic and real-world datasets.

\section{Synthetic Control (SC) Methods}\label{sec.background}

Before introducing SC methods, we introduce some key notation.
We denote a target unit as a vector $x$ (usually indexed $0$) and a donor pool as a matrix $X \in \mathbb{R}^{n \times T}$ with $n$ donor units (index ranging from $1$ to $n$) and $T$ observations.
Given a matrix $X$, let $X_i$
be the $i$-th row , $X_{i:j}$ be the submatrix constructed by choosing the rows between $i$-th and $j$-th rows, and $X_{i,t}$ be the element in the $i$-th row and the $t$-th column of $X$.
When donors are represented as a set of points, we use $x_i$ to denote the point corresponding to the $i$-th donor unit. 
Assuming an intervention at time $T_0<T$, $X$ can be split into pre-intervention portion $X^-\in \mathbb{R}^{n \times T_0}$ and post-intervention portion $X^+\in \mathbb{R}^{n \times (T-T_0)}$. Similarly for a vector, $x = [x^-, x^+]$ is the pre- and post-intervention split.
We denote the $i$-th singular value of a matrix $X$ by $\sigma_i(X)$ and the $i$-th eigenvalue of a square matrix $X$ by $\lambda_i(X)$.
If needed, we denote the left and right singular vectors of a matrix $X$ as $u_i(X)$ and $v_i(X)$, respectively.
We use $\|\cdot\|$ to denote the spectral norm for a matrix and $\ell_2$ norm for a vector.

\textbf{SC Family of Methods.}
Imagine that a new property tax policy was implemented in New York, but not in other states in the US.
The time series data would include $T$ observations of quarterly housing price index $x_{i,t} \in \mathbb{R}$ for all cities (units) $i \in V$  and for all time points $t\in [T]$. At time $T_0 < T$, only the cities in New York (treated units) $W \subset V$ adopt a new policy (intervention), while other cities outside of New York are not affected (control units, potential donors).
Hence, for each treated unit $i \in W$, we have a pre-intervention time series $x_{i}^- \in \mathbb{R}^{T_0}$ without intervention and post-intervention time series $x_{i}^+\in \mathbb{R}^{T-T_0}$ under the new policy. For a control unit $j \in V \setminus W$, we can use the same notation but the post-intervention time series $x_{j}^+$ was not affected by the intervention.

SC estimates the effect of an intervention on treated units in $W$ by constructing the counterfactual for the post-intervention period. It is important to note that SC constructs a separate model for each treated unit, allowing the causal estimand to be calculated on a per-unit basis. The SC family of methods learns the relationship between a target unit ($i=0$ from $W$) and donor units ($j= 1, \ldots, n$ from $V \setminus W$) using pre-intervention time series data. Assuming this relationship remains stable over time $t \in [T]$, the counterfactual post-intervention time series for the target unit is inferred using donor data from the post-intervention period. Algorithm \ref{alg.sc.family} formally defines the synthetic control family of methods.

\begin{algorithm}
\caption{Synthetic Control Family of Methods}
\begin{algorithmic}\label{alg.sc.family}
\STATE \textbf{Data: } Target time series vector $x_i \in \mathbb{R}^T$ for each treated unit $i \in W$.
Donor data $X \in \mathbb{R}^{n \times T}$ containing all control units $j \in V \setminus W$.
\FOR{$i \in W$}
\STATE \textbf{1. Learn }
$f = \mathcal{M}(X, x_i^-)$
\STATE \textbf{2. Project }
$\hat{m}_{i}^+ = f(X^+)$
\STATE \textbf{3. Infer } the estimated causal effect of the intervention for target $i$ is $x_i^+ - \hat{m}_{i}^+$
\ENDFOR
\end{algorithmic}
\end{algorithm}

In the first step of Algorithm \ref{alg.sc.family}, $\mathcal{M}$ learns weights $f$ to represent the target unit as a linear combination of the donor units.
In the original work on synthetic control, \cite{abadie2003economic} use linear regression with a simplex constraint on the weights (i.e., the regression coefficients should be non-negative and sum to one). They used data on per capita GDP in $n=17$ Spanish regions (aggregate level) to measure the effect of terrorism on Basque Country's per capita GDP.
Later, more advanced variations of synthetic control have been proposed to deal with multiple treated units \citep{dube2015pooling, abadie2021penalized}, to correct bias \citep{ben2021augmented, abadie2021penalized}, to remove simplex constraints \citep{doudchenko2016balancing, rsc}, to ensure differential privacy \citep{dpsc}, to incorporate matrix completion techniques \citep{athey2021matrix, mrsc}, and to consider temporal order \citep{brodersen2015inferring}.

In this paper, we will use Algorithm \ref{alg.sc.core} as our learning method $\mathcal{M}$, which is based on the method proposed by \cite{rsc}. It denoises the donor matrix by retaining only the top $r$ singular values through hard singular value thresholding (HSVT) \citep{cai2010singular, chatterjee2015matrix}, followed by ordinary least squares to obtain the weight vector $f$. This approach is known for its robustness to noisy data, making it well-suited to our objectives.

\begin{algorithm}
\caption{Learn Step of Synthetic Control Algorithm $\mathcal{M}(X, x^-; r)$
}
\begin{algorithmic}\label{alg.sc.core}
\STATE \textbf{Input: } donor data $X$, pre-intervention target data $x^-$, the number of singular values to keep $r$
\STATE \textbf{1. Perform SVD }
\STATE $X = \sum_{i=1}^{T} \sigma_i u_i v_i^{\top}$, $\sigma_i$ in decreasing order.
\STATE \textbf{2. Denoise }
$\hat{M} = \sum_{i=1}^r \sigma_i u_i v_i^{\top} \coloneqq HSVT(X;r)$.
\STATE \textbf{3. Return SC weights}
\STATE 
$\hat{f} = \argmin_{f\in \mathbb{R}^{n}} \|  \hat{M}^{-\top} f - x^-\|$ (SC weights)
\end{algorithmic}
\end{algorithm}

An intuitive way to view SC is a linear regression vertically performed on the dataset. The pre-intervention donor matrix $X^-$ is the regressor and the pre-intervention target time series $x_0^-$ is the regressand, so the $j$-th element of the weight vector $f$ represents the importance of the $j$-th donor unit in explaining the target unit $0$. Since a column of the matrix $X^-$ becomes one sample for learning, we call this a \emph{vertical regression}.

Another way to view SC is as a matrix completion problem with post-intervention target data as missing values.
\citet{athey2021matrix} formalizes SC as a matrix completion method by setting an objective function based on the Frobenius norm of the difference between the latent and the observed matrix.
The core modeling assumption of this approach is that the matrix is approximately low-rank.
This is achieved by assuming a Lipschitz-continuous latent variable model with bounded latent variables \citep{candes2010matrix, candes2012exact, nguyen2019low}.

\textbf{SC on Disaggreagate-level Data}
When applying SC to disaggregate-level data, meeting these assumptions becomes more challenging. For example, there might be a certain \emph{type} of units (such as patients with a certain phenotype) that can be well-approximated by a low-rank matrix, but not when mixed with other units in different types.
When the number of potential donors is small, it may be possible to hand-pick a suitable donor set based on background knowledge, which is usually done for aggregate-level datasets \citep{abadie2003economic, abadie2015comparative}.

However, with disaggregate-level data, researchers must devise more data-driven approaches to select the appropriate donor units for a given target. \citet{abadie2021penalized} used a penalty term to keep the \emph{active units} in the donor pool  small. Other works suggest using Lasso \citep{chernozhukov2021exact} or elastic net \citep{doudchenko2016balancing} regularizers to achieve similar results.
Still, these SC configurations operate in $n$-dimensional spaces, which is less feasible when $n$ is large.

\section{Problem Setup}\label{sec.setup}
In this paper, we focus on applying SC to disaggregate-level data. Given the abundance of donor units, our objective is to develop a pre-processing step for SC that selects the optimal set of donors for a given target unit. In the following subsections, we present a detailed model tailored to this setting.

To assess the performance of SC methods, researchers often construct a \emph{placebo} test \citep{abadie2003economic}, where SC is used to predict post-intervention data in the absence of an intervention, or equivalently, to predict the post-intervention time series of a control unit using other control units as the donor pool.  In these settings, since the target is drawn from the same distribution as the donors, the estimated causal effect (from Algorithm \ref{alg.sc.family}) should be zero. To more easily articulate the accuracy of SC methods, we focus on these placebo studies in the remainder of the paper.

\subsection{Model}\label{sec.model}
Let $x_0$ be the target unit and let $X \in \mathbb{R}^{n\times T}$ be the donor data matrix, where each row $x_i$ is a $T$-length time-series measurement.
In light of disaggregate-level data, we assume $n \gg T$ (i.e., $X$ is a tall matrix).
We assume that the true data generation model comes from a latent variable model, plus some observation noise. 
That is, $X = M + E$, where $M_{i,t}$ is the true (deterministic) signal with entries bounded $-1 \leq M_{i,t} \leq 1$, and $E_{i,t}$ is mean-zero noise with finite variance $s^2$, for all $i \in \{0, \cdots, n\}$ and $t \in [T]$.
Similarly, we assume the target $x_0 = m_0+\epsilon_0$ with zero-mean finite-variance ($s^2$) noise $\epsilon_0$.

Consistent with the synthetic control literature, we assume the entries of \(M\) are generated by a latent variable model, i.e., $M_{i,t} = g(\theta_i, \rho_t)$ where $\theta_i$ and $\rho_t$ are finite-dimensional latent vectors \citep{ben2021augmented, arkhangelsky2021synthetic, abadie2021using, rsc, mrsc, athey2021matrix}. Since we focus on placebo studies, we assume this holds for the target unit as well: $m_{0,t} = g(\theta_0, \rho_t) \; \forall t \in [T]$.\footnote{In the case of target unit that experienced an intervention, this would not necessarily hold for $t>T_0$.} We assume \(g\) is $L$-bilipschitz continuous, so that cluster structure in the latent variables is recoverable by our algorithm (see Section \ref{sec.theory.subgroup}).
Then, $M$ is known to be well-approximated by a low-rank matrix \citep{chatterjee2015matrix} with \(\rank(M) = O(\log T)\) \citep{udell2019big}. We denote \(\rank(M) =r\) and assume \(r < T\).
Finally, we assume that there exists a vector $f^*$ with $\|f^*\| \leq \mu$ for some $\mu>0$, satisfying $M_{0,t} = M_{1:n, t}^\top f^*$.

\subsection{Existence of Subgroups}\label{sec.subgroup}

Our motivation comes from the idea that the donors may have some relevant subgroups in the population or underlying cluster structure, and the target unit belongs to one of these clusters.
We formalize this by assuming a centroid-based separation structure on the row latent variables $\Theta = \{\theta_i:i\in [n]\}$. 
Let $P = \{P_j\}_{j\in [k]}$ be a \(k\)-partition of $\Theta$ (i.e., \(P_1\sqcup ...\sqcup P_k = \Theta\)), with induced centers $\{c_j\}_{j\in [k]} = \{\frac{1}{|P_j|}\sum_{\theta\in P_j} \theta\}_{j\in [k]}.$
Then, we define the $k$-means cost
$\Delta_k^2(\Theta; P) = \sum_{i=1}^n \min_{j\in [k]}\|\theta_i - c_j\|^2$, which captures the average distance from the cluster center to members of the cluster. We denote \(\Delta_k^2(X) = \min_{P} \Delta_k^2(X; P)\) as the cost of an optimal $k$-means solution of input $X$.
Finally, we assume the $\varepsilon$-separation condition on $\Theta$, as in Definition \ref{def.epsseparation}.
\begin{definition}[$\varepsilon$-separation]\label{def.epsseparation}
We say that $\Theta  = (\theta_1,...,\theta_n)\subset \mathbb{R}^d$ is $\varepsilon$-separated with $k$ clusters if for some integer $k \geq 2$ and $\varepsilon \in (0,1)$,
\begin{equation}\label{eq.assump.separation}
\Delta_k^2(\Theta) \leq \varepsilon^2\Delta_{k-1}^2(\Theta).
\end{equation}
\end{definition}
This captures the idea that \(k\) clusters fit the data significantly better than \(k-1\) clusters (in the spirit of the ``elbow method'' heuristic).
For example, this condition would be satisfied if $\Theta$ were generated by a sufficiently separated mixture of $k$ distributions.
This modeling assumption on the existence of subgroups provides us with a formal setup to show the ability of our algorithm to approximate the clusters in the latent variable space.

\section{Cluster Synthetic Control (ClusterSC) Algorithm}\label{sec.algo}
In this section, we present Cluster Synthetic Control (ClusterSC, Algorithm \ref{alg.csc}), which integrates a donor-clustering step into the synthetic control framework. The clustering subroutine is designed to identify structural patterns within the donor pool, ensuring that units within the same cluster exhibit similar behavior while differing across clusters. Given a target unit, the algorithm selects the most relevant cluster, after which synthetic control is applied using only the chosen donors.

The core intuition behind our ClusterSC algorithm is that using more donor units corresponds to higher-dimensional inputs in the linear regression step of synthetic control, which in turn leads to higher-dimensional noise and more instability. Therefore, we want to restrict to only the most relevant donors (via clustering) and thereby lower the dimension of the regression.
We accomplish this in a two-step approach:

\begin{itemize}
    \item Algorithm \ref{alg.cluster} is a clustering step that partitions the donor units using \(k\)-means clustering. This enables the donor selection algorithm to later identify the correct donor cluster for the target unit. 
    \item With the identified clusters, Algorithm \ref{alg.csc} finds a matching subgroup for a given target and performs SC using only this subgroup as donor units. This subgroup specialization leads to a more accurate SC predictions with lower computational cost.
\end{itemize}

\begin{algorithm}[h]
\caption{Clustering Algorithm $\mathcal{C}(X; r)$}
\begin{algorithmic}\label{alg.cluster}
\STATE \textbf{Input: } Donor matrix $X$, approximate rank $r$
\STATE \textbf{1. Perform SVD}
\STATE $X = U \Sigma V^\top$ 
\STATE $\tilde{M} = U \Sigma_r V^\top :=HSVT_r(X) \quad$ \# Hard Singlular Value Thresholding
\STATE $\tilde{U} = U \Sigma_r \quad$ \# Features used for clustering
\STATE \textbf{2. Perform $k$-means clustering} \(n^{O(1)}\) steps of Lloyd's method on the rows of $\tilde{U}$.
\STATE \textbf{3. Return} cluster centers and $V$
\end{algorithmic}
\end{algorithm}

Algorithm \ref{alg.cluster} uses the assumption that the signal matrix $M$ is low-rank with rank $r$, hence the noisy version $X$ is approximately low-rank. It first performs a singular value decomposition (SVD) $X = U \Sigma V^\top = \sum_{i=1}^r s_i u_i v_i^\top$, where the $v_i$'s represent the $r$ basis row vectors; that is, any rows in $X$ can be expressed as a linear combination of the $v_i$'s.
Define $\tilde{U}=U\Sigma$.
Then, for the $j$-th \emph{row} of $\tilde{U}$, $\tilde{U}_{j,i}$ can be interpreted as the number of basis vectors $v_i$ that are used to describe row $X_j$.

\begin{wrapfigure}{l}{0.5\textwidth}
  \begin{center}
\includegraphics[width=\textwidth]{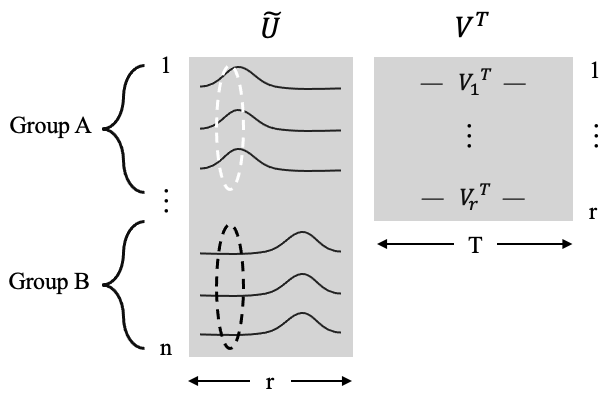}
  \end{center}
\caption{Visualization of the distribution of rows in $\tilde{U}$ with two different subgroups in the donor units. Each row $\tilde{U}_i$ can be interpreted as an embedding of the unit $i$, representing the composition of right singular vectors for that row. 
}\label{fig.u.tilde}
\end{wrapfigure}

What if there are two groups with very different distributions of $U_j$'s? Figure \ref{fig.u.tilde} visualizes this difference with the white dashed circle and the black dashed circle illustrating the differences in groups A and B in a single column: for each column of $\tilde{U}$, we expect the distribution to be different across groups but similar within each group.
Performing a clustering algorithm (e.g., $k$-means in our case) on $\tilde{U}$, we are separating units based on the use of right singular vectors $v_i$.

Our ClusterSC in Algorithm \ref{alg.csc} incorporates Algorithm \ref{alg.cluster} as a pre-processing step on the donor pool. The second step of Algorithm \ref{alg.csc} computes $\tilde{u}$, a counterpart of $\tilde{U}$ for the target, and finds the matching cluster for the target.  A donor matrix $A$ is then constructed using data from the selected cluster. Note that only pre-intervention data are used for these steps, based on the motivating use-case of SC, where post-intervention data are assumed to be missing for the target unit.

For the Learn step of 
synthetic control algorithm within Algorithm \ref{alg.csc}, we adopt Algorithm \ref{alg.sc.core} since de-noising with HSVT before the regression has been shown to be robust to noise \citep{rsc}, although other preferred methods could be used in this step instead. This SC step takes the selected sub-group donor matrix $A$ as an input, instead of the whole donor set $X$.

\begin{algorithm}
\caption{ClusterSC $(X, x_0; r)$}
\begin{algorithmic}\label{alg.csc}
\STATE \textbf{Input: } Donor matrix $X$, target data $x_0$, approximate rank $r$.
\STATE \textbf{1. Learn clusters}
\STATE $c_1, \cdots, c_k$, $V \leftarrow \mathcal{C}(X;r)$ (Algorithm \ref{alg.cluster}, $c_t$ are cluster centers)
\STATE \textbf{2. Find target's matching donor cluster $t$}
\STATE $\tilde{u} = V^{-\top} x_0^-$
\STATE $t = \arg\min_{t'} \|c_{t'}-\tilde{u}\|_2$ ($t$ is target's cluster label)
\STATE \textbf{3. Construct donor matrix $A$ and denoise}
\STATE $A = X_{C_t}$ ($C_t$ is the set of units in cluster $t$)
\STATE $\hat{M}_{C_t} = HSVT(A; r)$ (Denoise selected donor)
\STATE \textbf{4. SC: Learn}
$\hat{f} \leftarrow \mathcal{M}(A, x_0^-; r)$ (Algorithm \ref{alg.sc.core})
\STATE \textbf{5. SC: Project}
$\hat{m}_0^+ \leftarrow \hat{f}(\hat{M}_{C_t}^+)$ 
\STATE \textbf{6. SC: Infer } the estimated causal effect of the intervention for the target is $x_0^+ - \hat{m}_{0}^+$
\end{algorithmic}
\end{algorithm}

\section{Theoretical Guarantees}\label{sec.theory}

In this section, we provide theoretical guarantees on the performance of ClusterSC (Algorithm \ref{alg.csc}) by showing the accuracy of identifying subgroups (Section \ref{sec.theory.subgroup}) and the impact of subgroup specialization via ClusterSC on the prediction accuracy (Sections \ref{sec.changes} and \ref{sec.theory.improve}).

Following the notation introduced in Section \ref{sec.model}, let $X = M+E_M$ be a $n \times T$ donor pool matrix and $A = S+E_S$ be a sub-matrix constructed by taking $n_A$ rows of $X$ based on the ClusterSC output.
Let the low-rank signal matrices have $\rank(M)=r$ and $\rank(S) =r_S$.
Then, we say the approximate-rank of $X$ is $r$, and we define the $(r+1)$-th singular value of a matrix $X$ as $\sigma^*_X = \sigma_{r+1}(X)$. 

The pre-intervention mean squared error of synthetic control estimator
is given by:
$
   \MSE(\hat{m}^-;X) = \mathbb{E}[\frac{1}{T_0}\| m^- - \hat{M}^{-\top}\hat{f}\|^2],
$
where \(\hat{f}\leftarrow \mathcal{M}(X, x^-)\).
Likewise, the post-intervention error is 
$
   \MSE(\hat{m}^+;X) = \mathbb{E}[\frac{1}{T-T_0}\| m^+ - \hat{M}^{+\top}\hat{f}\|^2].
$
RMSE is defined by taking a squared root inside the expectation of either expression.
We are interested in the change in \(\MSE\) (or RMSE) when replacing \(X\in \mathbb{R}^{n\times T}\) with its subset \(A\in \mathbb{R}^{n_A\times T}\).

\subsection{Accuracy of Subgroup Identification}\label{sec.theory.subgroup}

In this section, we show that existing subgroups in $\Theta$-space are well-approximated by Algorithm \ref{alg.cluster}.
We need to show that the cluster structure assumed only in the $\Theta$ space will continue to hold in $\tilde{U}$ space, where the clustering is performed in Algorithm \ref{alg.cluster}. To do so, we first show that this structure is well-preserved in $M$ via bilipschitz mapping (Section \ref{sec.theta.m}).
Then, we show $\tilde{M}$, hard singular value thresholding (HSVT) applied to $X=M+E$, is close to the signal matrix $M$ (Section \ref{sec.m.mtilde}).
Finally, we show that clustering with $\tilde{U}$ features as in Algorithm \ref{alg.cluster} can well-approximate clusters in $\tilde{M}$ (Section \ref{sec.mtilde.utilde}). All omitted proofs from this section are presented in Appendix \ref{app.proofs}.

Some new notation is needed to discuss the clustering results.
To measure the accuracy of approximation, we say partition $P^A$ is $\epsilon$-approximated by partition $P^B$ if the two partitions agree with each other for all but $\epsilon$ fraction of the points. We use $A \ominus B = (A \setminus B)\cup (B \setminus A)$ to denote the symmetric difference between sets $A$ and $B$.
For a set of points $A={a_1, \ldots, a_n}$, we define the $k$-mean optimal cluster centers $C^A = \{c^A_i\}_{i=1}^{k}$ and the induced Voronoi partition $P^A = \{P^A_i\}_{i=1}^{k}$.
The optimal $k$-means objective is defined as
$
\Delta_k^2(A)
= \sum_{i\in[n]} \min_{j\in[k]} ||a_i-c^A_j||^2 
= \sum_{l\in[k]} \frac{1}{2|P^A_l|}\sum_{i,j \in P^A_l} ||a_i-a_j||^2.
$
When a partition $\hat{P}$ is specified, the optimal $k$-means objective can be written as 
$
\Delta_k^2(A; \hat{P})
= \sum_{l\in[k]} \frac{1}{2|\hat{P}_l|}\sum_{i,j \in \hat{P}_l} ||a_i-a_j||^2,
$
with the centers recalculated as a mean of each partition.

The main result of this section is Theorem \ref{thm.cluster_consistency}, which combines all three steps (Lemmas \ref{lem.cluster_1}, \ref{lem.cluster_2}, and \ref{lem.cluster_3}) to show that Algorithm \ref{alg.cluster} well-approximates the optimal $k$-means $P^\Theta$.
Each lemma bounds the symmetric difference between the partitions $P^\Theta$, $P^M$, $P^{\tilde M}$ and the partition learned by clustering subroutine (Algorithm \ref{alg.cluster}).
By summing up the differences in each step, we can guarantee that the initial partition $P^\Theta$ and the algorithmic output will disagree by at most half of the sum of the symmetric difference.

\subsubsection{Bilipschitz mapping of cluster structure in $\Theta$ to $M$}\label{sec.theta.m}

Recall from Section \ref{sec.model} that $\Theta$ is the unobserved latent variables, and $M = g(\Theta)$ is the signal created from $\Theta$ via $L$-bilipschitz mapping $g$. Using the $\varepsilon$-separation structure assumed in $\Theta$, we investigate how the clustering structure in $\Theta$ is preserved in $M$.
Our analysis involves an intermediate labeling $P'$, constructed as follows.
Take the labeling under $P^\Theta$ and consider the distribution of points in $M$ space. Let \(c' = (c_1',...,c_k')\) be the centers induced by \(P^\Theta\) on \(M\), i.e., $c'_i =\frac{1}{|P_i^\Theta|}\sum_{l\in P_i^\Theta} m_l$, and $P'$ be the Voronoi partition induced by $c'$. 
We will use $P'$ as an intermediate step to show that $P^\Theta$ and $P'$ are similar using the $\varepsilon$-separation structure (Definition \ref{def.epsseparation}) of $\Theta$. We will then show that $P'$ and $P^M$ are similar using Theorem \ref{thm.ostrovsky}, which states that if $P'$ yields sufficiently small $k$-means cost on $M$, and $M$ is sufficiently separated, then $P'$ is similar to $P^M$, the optimal $k$-means result on $M$. Note that in the case where $P^\Theta$ on $M$ forms a Voronoi partition, then $P'=P^\Theta$, and the first step is not needed.

To apply Theorem \ref{thm.ostrovsky} later, we first require Lemma \ref{lem.bilipseffect}, which gives the relationship between the optimal $k$-means objective with input $M=g(\Theta)$ and with input $\Theta$.

\begin{restatable}{lemma}{bilipseffect}\label{lem.bilipseffect}
For any $L$-bilipschitz function $g$, $(1/L^2)\Delta^2_k(\Theta) \leq \Delta^2_k(g(\Theta)) \leq L^2 \Delta^2_k(\Theta)$.
\end{restatable}

Using this result, we can show in Lemma \ref{lem.kmeanscostbound} that the conditions for Theorem \ref{thm.ostrovsky} are satisfied, namely that $M$ is $L^2\varepsilon$-separated and the $k$-means cost of partition $P^\Theta$ calculated on $M$ is bounded by $L^4\varepsilon^2 \Delta_{k-1}^2(M)$.

\begin{restatable}{lemma}{kmeanscostbound}\label{lem.kmeanscostbound}
$M=g(\Theta)$ is $L^2\varepsilon$-separated with $k$ clusters, and $\Delta^2_k(M;P^\Theta) \leq L^4 \varepsilon^2 \Delta^2_{k-1}(M)$.
\end{restatable}

\begin{proof}
We first show the $L^2 \varepsilon$-separation of $M$ by comparing the the cost of clustering $M=g(\Theta)$ with $k$ and $k-1$ clusters.
\begin{align*}
    \Delta^2_k(M)
    &\leq L^2 \Delta^2_k(\Theta) &&\text{(Lemma \ref{lem.bilipseffect})} \\
    &\leq L^2 \varepsilon^2 \Delta^2_{k-1}(\Theta) && \text{(Modeling assumption of Eq \eqref{eq.assump.separation})} \\
    &\leq L^4\varepsilon^2 \Delta^2_{k-1}(M). && \text{(Lemma \ref{lem.bilipseffect})}
\end{align*}
To proof the second claim of the lemma, we observe that the proof of Lemma \ref{lem.bilipseffect} shows that the first line decomposes to
\[
\Delta^2_k(M) \leq \Delta^2_k(M; P^\Theta) \leq L^2 \Delta^2_k(\Theta),
\]
as shown in Equations \eqref{eq.kmeans.bounding.2}, \eqref{eq.kmeans.bounding.intermediate}, and \eqref{eq.kmeans.bounding.3} of the proof.
Taking the second term and again applying Lemma \ref{lem.bilipseffect}, we obtain that $\Delta^2_k(M; P^\Theta) \leq L^4\varepsilon^2 \Delta^2_{k-1}(M)$.
\end{proof}

With this in mind, we now focus on bounding the difference between $P^\Theta$ and $P'$.
First, define
$r_i^2(A):=  \frac{1}{|P_i^A|}\sum_{l\in P_i^A} \|A_l - c_i^A\|^2$, the mean squared error of cluster $P^A_i$, for any set of points $A$. Then, \cite{ostrovsky2013effectiveness} gives a useful lemma to bound these errors.

\begin{lemma}[\citep{ostrovsky2013effectiveness}]\label{lem.ri.bound} Let $\Theta$ be $\varepsilon$-separated. Then, for every $i \in [k]$, we have 
$r_i^2(\Theta) \leq \frac{\varepsilon^2}{1-\varepsilon^2} \min_{j\neq i} \|c_i^\Theta - c_j^\Theta\|^2.$
\end{lemma}

Using this, we can bound the distance between the centers $c'$ and the bilipschitz-map of centers in $\Theta$, $g(c^\Theta)$.

\begin{restatable}{lemma}{pprime}\label{lem.p.prime.1}
For all \(i\in [k]\), \(\|g(c_i^\Theta) - c'_i\| \leq L \cdot r_i(\Theta)\). 
\end{restatable}

Now, we define $core(P^\Theta_i)$, a core set that contains at least a $1-\varepsilon$ fraction of the points in partition $P^\Theta_i$.

\begin{restatable}{lemma}{coreset}\label{lem.core.set}
    Let $core(P^\Theta_i) := \{l\in P_i^\Theta: \|\theta_l - c_i^\Theta\| \leq \sqrt{\frac{\epsilon}{1-\epsilon^2} }\min_{j\neq i} \|c_i^\Theta - c_j^\Theta\|\}.$ Then, for all $i\in[k]$, $|core(P^\Theta_i)| \geq (1-\varepsilon)|P^\Theta_i|$.
\end{restatable}

Lastly, we show a useful characteristic of the set $core(P^\Theta_i)$.

\begin{restatable}{lemma}{pprimeagain}\label{lem.p.prime.2}
Choose two distinct partitions $P^\Theta_i$ and $P^\Theta_j$. Then, for all \(l\in core(P^\Theta_i) \),
    \[\|\theta_l - c_j^\Theta\| - \|\theta_l - c_i^\Theta\| 
    \geq \left( 1 - 2\sqrt{\frac{\varepsilon}{1-\varepsilon^2}} \right) \|c_j^\Theta - c_i^\Theta\|.\]
\end{restatable}

Combining Lemmas \ref{lem.p.prime.1}, \ref{lem.core.set}, and \ref{lem.p.prime.2}, we show that $P'$ and $P^\Theta$ are similar.

\begin{restatable}{lemma}{pprimesmall}\label{lem.p.prime.small}
For small $\varepsilon \leq 0.1$, if $L^2<\frac{\sqrt{1-\varepsilon^2}+\sqrt{\varepsilon}}{2\varepsilon+3\sqrt{\varepsilon}}$, 
then $\sum_{i=1}^k|P_i' \ominus P^\Theta_i| \leq 2\varepsilon n$, where $n=|\Theta$ is the number of donor rows.
\end{restatable}

Finally, we bound the difference between $k$-means optimal clusters in $\Theta$ and in $M$.
Lemma \ref{lem.cluster_1} uses Lemma \ref{lem.p.prime.small} to show $P^\Theta \approx P'$ and instantiates Theorem \ref{thm.ostrovsky} to show $P' \approx P^M$.
Again, if $P^\Theta$ on $M$ forms a Voronoi partition, then $P^\Theta = P'$ and the error from that step of the analysis would be zero.
Note that the bound on $L$ depends on $\varepsilon$; we need approximately $\varepsilon<0.1$ to ensure $L>1$, and the upper bound on $L$ increases as $\varepsilon$ diminishes.

\begin{restatable}{lemma}{clustconsist}\label{lem.cluster_1}
For small $\varepsilon \leq 0.1$, 
if $L^2 \leq \min(\frac{1}{\sqrt{801}\varepsilon}, \frac{\sqrt{1-\varepsilon^2}+\sqrt{\varepsilon}}{2\varepsilon+3\sqrt{\varepsilon}} )$,
then $\sum_{i=1}^k|P^\Theta_i \ominus P^M_{\sigma(i)}| \leq 8L^2\varepsilon n$ for some bijection $\sigma(i)$ and where $n=|\Theta|$.
\end{restatable}

\subsubsection{Approximating $M$ with $\tilde{M}$}\label{sec.m.mtilde}
Next, we show the difference between the points represented by $M$ and $\tilde{M}$, where $\tilde{M} = HSVT(X; r) = \sum_{i=1}^r \sigma_i u_i v_i^\top$ where $\sigma_i$, $u_i$, and $v_i$ are respectively the $i$-th singular value, left singular vector, and right singular vector of $X=M+E$.
To quantify the difference between $M$ and $\tilde{M}$, we define $\eta := \max_{i \in [n]} \| m_i - \tilde{m}_i\|$, where $m_i$ and $\tilde{m}_i$ are respectively the $i$-th rows of $M$ and $\tilde{M}$.

Define $G$ as a set of \emph{good} events where the noise is small enough that $\eta \leq \frac{2s(\sqrt{n} + \sqrt{T})}{\delta}$ for some $\delta>0$, which happens with high probability (at least $1-\delta$).

\begin{restatable}{lemma}{etabound}\label{lem.eta.bound}
Let $\eta = \max_{i \in [n]} \| m_i - \tilde{m}_i\|$. Then with probability $1-\delta$, 
$\eta \leq \frac{2s(\sqrt{n} + \sqrt{T})}{\delta}$.
\end{restatable}

With Lemma \ref{lem.eta.bound}, we can focus only on the events in $G$ for the remainder of the proof. 
Lemma \ref{lem.mtilde.sep} then shows that conditioned on $G$, a separation structure is preserved in $\tilde{M}$ with a scaled factor.

\begin{lemma}\label{lem.mtilde.sep}
Choose $\delta\in (0,1)$.
If $s < \frac{\delta L^2\varepsilon \Delta_{k-1}(M)}{2\sqrt{n}(\sqrt{n}+\sqrt{T})}$ and $L^2\varepsilon < \frac{1}{2}$, then conditioned on the good event $G$, $\tilde{M}$ is $4L^2\varepsilon$-separated.
\end{lemma}
\begin{proof}
By Lemma \ref{lem.eta.bound} and the assumption on $s$, we have
\[
\max_i \|\tilde m_i - m_i\| 
\leq \frac{2s(\sqrt{n} +\sqrt{T})}{\delta}
\leq \frac{L^2\varepsilon\Delta_{k-1}(M)}{\sqrt{n}}
\]
with probability $1-\delta$.
Then, conditioned on this event occurring, by Theorem \ref{thm.ostrovsky.ii}, $\tilde{M}$ is 
$\sqrt{\frac{8L^4\varepsilon^2}{1-2L^4\varepsilon^2}}$-separated.
Under the mild assumption of $L^2\varepsilon < \frac{1}{2}$, we can bound
$\sqrt{\frac{8L^4\varepsilon^2}{1-2L^4\varepsilon^2}}
\leq 4L^2\varepsilon$.
We thus conclude that $\tilde{M}$ is $4L^2\varepsilon$-separated.
\end{proof}

Next, we define an intermediate step $\tilde P$ to connect $P^M$ and $P^{\tilde{M}}$.
Let $c_i^M = \frac{1}{|P_i^M|}\sum_{l\in P_i^M} m_l $ be the $k$-means optimal centers of the points represented by $M$.
Define $\tilde P$ as the partition generated by $\{c_i^M\}_{i\in [k]}$ on $\tilde M$.
It is possible both for the membership of points in $\tilde P$ to change in $P^M$, and for the re-calculated centers 
$c^{\tilde{M}}$ to additionally introduce a difference between $\tilde P$ and $P^{\tilde M}$.
The next two lemmas address these changes.

\begin{restatable}{lemma}{noisyintone}\label{lem.noisy.intermediate1}
For any $\delta>0$, if
$s < \frac{ \delta \left(-\sqrt{16T} + \sqrt{16T + \frac{ 6\varepsilon^2 L^4 \Delta^2_{k-1}(\tilde M)}{n} } \right)} {12(\sqrt{n}+\sqrt{T})}$
and \(L^2\varepsilon < \frac{1}{2}\), then in the event of $G$, \(|\tilde P\ominus P^{\tilde M}| \leq 2576 L^4\varepsilon^2 n\).
\end{restatable}

\begin{restatable}{lemma}{noisyinttwo}\label{lem.noisy.intermediate2}
For any $\delta>0$, if \(s \leq \frac{\delta \Big(1-2\sqrt{\frac{L^2\varepsilon}{1-L^4\varepsilon^2}}\Big)\min_{i\neq j} \|c_i^M - c_j^M\|}{4(\sqrt{n} + \sqrt{T})}\), then in the event of $G$, \(|\tilde P \ominus P^M| \leq 2 L^2 \varepsilon n\)
\end{restatable}

Finally, we combine Lemmas \ref{lem.noisy.intermediate1} and \ref{lem.noisy.intermediate2} to bound the difference between $P^{\tilde M}$ and $P^M$. Specifically, Lemma \ref{lem.cluster_2} shows that adding observation noise $E$ and then using HSVT to denoise does not substantially change the optimal $k$-means partition of $M$.

\begin{restatable}{lemma}{cluster}\label{lem.cluster_2}
    For $\varepsilon<0.1$ and $\delta \in (0,1)$, if \(L^2\epsilon < 1/\sqrt{801}\),
    $\min_ir_i(M)\geq1/160$,
    and $s<O(\frac{\delta \sqrt{T}}{\sqrt{n}+\sqrt{T}})$,
then, conditioned on $G$, $\sum_{i=1}^k|P^M_i \ominus P^{\tilde M}_{\sigma(i)}| \leq 94L^2\varepsilon^2n$ for some bijection $\sigma(i)$.
\end{restatable}

\subsubsection{Approximating clusters in $\tilde{M}$ using $\tilde{U}$}\label{sec.mtilde.utilde}

Finally, we show the approximation error of the output of Algorithm \ref{alg.cluster} with respect to $P^{\tilde M}$ is small.
Note that $\tilde M$ is an approximation of $M$ via $HSVT(X)$.
We continue to condition on the events in the good set $G$, where $4L^2\varepsilon$-separation is guaranteed in $\tilde M$ (Lemma \ref{lem.mtilde.sep}).
Then, we will translate the $\varepsilon$-separation condition to the proximity condition introduced in \cite{kumar2010clustering}. This proximity condition is defined as below.
\begin{definition}[Proximity Condition \citep{kumar2010clustering}]\label{def.proximity}
Let $X \in \mathbb{R}^{n \times T}$ be the data matrix where the rows $X_i$ are divided into $k$ clusters $P_1, \ldots, P_k$ with corresponding cluster centers $c_1, \ldots, c_k$.
Let $C$ be the $n$ by $T$ matrix with each row $C_i$ as the cluster center of which $X_i$ belongs to (i.e., $c_{\pi(X_i)}$ where $\pi$ denotes a function that coutputs the cluster of $X_i$.)
Define
\[
\Delta_{r,s} = \left( \frac{ck}{\sqrt{|P_r|}} + \frac{ck}{\sqrt{|P_s|}}\right) ||X-C||,
\]
where $c$ is a large enough constant.
We say a point $X_i \in P_r$ satisfies the proximity condition if for any $s \neq r$, the projection of $X_i$ onto the line connecting $c_r$ to $c_s$ is at least $\Delta_{r,s}$ closer to $c_r$ than to $c_s$.
\end{definition}

Theorem \ref{thm.kk.claim} (from \cite{kumar2010clustering}) shows that the $\varepsilon$-separation condition implies that at least a $1-\varepsilon^2$ fraction of points satisfy the proximity condition.
Then, we can apply Theorem \ref{thm.kk.theorem} --- which shows that if at least a $1-\varepsilon^2$ fraction of points satisfy the proximity condition, then all but $O(k^2 \epsilon n)$ points will be correctly partitioned --- to show that our Algorithm \ref{alg.cluster} well-approximates $P^{\tilde M}$.

\begin{lemma}\label{lem.cluster_3}
Let $\hat{P}$ be the partition learned by Algorithm \ref{alg.cluster}.
Conditioned on $G$,
$\sum_{i=1}^k|P_i^{\tilde M} \ominus \hat{P}_{\sigma(i)}| = O(k^2 L^4 \varepsilon^2 n)$ for some bijection $\sigma(i)$.
\end{lemma}
\begin{proof}
Conditioning on the event of $G$, $\tilde M$ is $4L^2\varepsilon$-separated, and thus all but an $16L^4\varepsilon^2$-fraction of points satisfy the proximity condition.
By instantiating
Theorem \ref{thm.kk.theorem}, then
Algorithm \ref{alg.cluster} can correctly classify all but $16 k^2 L^4\varepsilon^2 n = O(k^2 L^4 \varepsilon^2 n)$ points with respect to $P^{\tilde M}$ in polynomial time.
\end{proof}

Finally, we can prove our main result, Theorem \ref{thm.cluster_consistency}, which shows that the $k$-mean optimal partition in $\Theta$ is well-approximated by the output of Algorithm \ref{alg.cluster}.

\begin{theorem}\label{thm.cluster_consistency}
For $\varepsilon<0.1$ and $\delta \in (0,1)$, 
if $L^2 \leq \min(\frac{1}{\sqrt{801}\varepsilon}, \frac{\sqrt{1-\varepsilon^2}+\sqrt{\varepsilon}}{2\varepsilon+3\sqrt{\varepsilon}} )$,
    $\min_ir_i(M)\geq1/160$,
    and $s<O(\frac{\delta \sqrt{T}}{\sqrt{n}+\sqrt{T}})$,
then with probability \(1-\delta\), we have $\sum_{i=1}^k|P^\Theta_i \ominus \hat{P}_{\sigma(i)}| = O(k^2 L^2 \varepsilon n)$ for some bijection $\sigma(i)$.
\end{theorem}
\begin{proof}
To invoke Lemmas 
\ref{lem.cluster_1}, 
\ref{lem.cluster_2}, and
\ref{lem.cluster_3},
we require the following mild conditions:
\begin{itemize}
    \item The separation parameter $\varepsilon\in (0,1)$ is small enough:  $\varepsilon<0.1$
    \item The bilipschitz parameter $L\in (1, \infty)$ is bounded:
    \[
    L^2 \leq \min(\frac{1}{\sqrt{801}\varepsilon}, \frac{\sqrt{1-\varepsilon^2}+\sqrt{\varepsilon}}{2\varepsilon+3\sqrt{\varepsilon}} )
    \]

Note that this upper bound on $L^2$ goes to infinity as $\varepsilon$ becomes smaller, thus virtually removing the bound on $L$ when $\varepsilon$ is sufficiently small. When $\varepsilon$ is bigger than $\sim 0.011$, the first term dominates.

    \item The standard deviation of noise $s>0$ is bounded above: $s = O(\frac{\delta\sqrt{T}}{\sqrt{n}+\sqrt{T}})$

    \item A good event in $G$ occurs, which we know from Lemma \ref{lem.eta.bound} will happen with probability $1-\delta$.
\end{itemize}

With these reasonable conditions on parameters, we can combine Lemmas 
\ref{lem.cluster_1}, 
\ref{lem.cluster_2}, and
\ref{lem.cluster_3} to respectively bound $\sum_{i=1}^k |P^\Theta_i \ominus P^M_{\sigma_1(i)}|$, $\sum_{i=1}^k  |P^M_{\sigma_1(i)} \ominus P^{\tilde M}_{\sigma_2(i)}|$, and $\sum_{i=1}^k |P_{\sigma_2(i)}^{\tilde M} \ominus \hat{P}_{\sigma_3(i)}|$:
\begin{align*}
\sum_{i=1}^k |P_i^{\Theta} \ominus \hat{P}_{\sigma(i)}|
&\leq \sum_{i=1}^k \left( |P^\Theta_i \ominus P^M_{\sigma_1(i)}| + |P^M_{\sigma_1(i)} \ominus P^{\tilde M}_{\sigma_2(i)}| + |P_{\sigma_2(i)}^{\tilde M} \ominus \hat{P}_{\sigma_3(i)}|\right)\\
&\leq 8L^2 \varepsilon n + 94L^2\varepsilon^2 n +O(k^2 L^4 \varepsilon^2 n)\\
&\leq O(k^2 L^2 \varepsilon n).
\end{align*}
This last step holds because $L^2 \epsilon<1$.
\end{proof}

\subsection{Effects of Subgroup Specialization}\label{sec.changes}

The previous section
shows that the clustering subroutine (Algorithm \ref{alg.cluster}) in ClusterSC well-approximates the subgroup structure in $\Theta$. Next we'll 
analyze the changes in the synthetic control pipeline when ClusterSC (Algorithm \ref{alg.csc}) is used, which selects only donors from the target cluster $A$, instead of the classical SC, which uses the whole donor pool $X$.

With $k\geq 2$ clusters, we expect the following changes that will affect the prediction performance guarantees:
\begin{enumerate}
    \item The number of donor units shrinks, $n_A < n$. \vspace{-1mm}
    \item The rank of the signal matrix shrinks, $r_S \leq r$. \vspace{-1mm}
    \item The largest singular value is suppressed in the HSVT step, $\sigma^*_A < \sigma^*_X$. (Recall $\sigma^*_X=\sigma_{r+1}(X)$)  
\end{enumerate}

While the first two are trivial to see, the third one is not. In this subsection we provide analyses of the gap $\sigma^*_X - \sigma^*_A$ under three different noise settings: Gaussian, sub-Gaussian, and heavy-tailed. 

\subsubsection{Gaussian Noise Setting}\label{s.gaussiannoise}

Theorem \ref{thm.singval.gauss} presents our first result on the singular values, under Gaussian noise $E_{i,t} \sim \mathcal{N}(0, s^2)$. This result shows that the gap between $\sigma^*_X$ and $\sigma^*_A$ will grow with the scale of noise $s$.

\begin{theorem}[Singular Value Concentration with Gaussian Noise]\label{thm.singval.gauss}
Let the noise terms be sampled $E_{i,t} \sim \mathcal{N}(0, s^2)$. If $r<T$ and $n_A < n + 4T -4 \sqrt{nT}$, then
\[\mathbb{E}[\sigma^*_X - \sigma^*_A]  \geq s (\sqrt{n}-\sqrt{n_A} - 2\sqrt{T}).\]
\end{theorem}
\begin{proof}
First, we show that $\mathbb{E}[\sigma^*_A] \leq s(\sqrt{n_A} + \sqrt{T})$.
\begin{align*}
    \mathbb{E}[\sigma^*_A] & =  \mathbb{E}[\sigma_{r_S+1}(S+E_S)] \\
    &\leq \sigma_{r_S+1}(S) + \mathbb{E}[\sigma_1 (E_S)]\\
    & \leq s(\sqrt{n_A} + \sqrt{T}).
\end{align*}
The first inequality is from Weyl's inequality
on singular values (Theorem \ref{thm.weyl.sing}), and $\sigma_{r_S+1}(S)=0$ by construction. The second inequality is from Gordon’s theorem: $\mathbb{E}[\sigma_1(E_s)]\leq s(\sqrt{n_A}+\sqrt{T})$ \citep{vershynin2010introduction}.

Next, we show $\sigma^*_X \geq s(\sqrt{n} - \sqrt{T})$ by analyzing the eigenvalues of $X^\top X$:
\begin{align*}
    \lambda_{r+1}(X^\top X) &= \lambda_{r+1}(M^\top M + 2M^\top E + E^\top E) \\
    &\geq \lambda_{r+1}(M^\top M) +\lambda_{T}(2M^\top E) +\lambda_{T}(E^\top E)\\
    &= \lambda_{T}(E^\top E)
\end{align*}
The first line is from the definition of $X=M+E$, the second line is from Weyl's inequality on eigenvalues (Theorem \ref{thm.weyl.eigen}) and the third line is because $\lambda_{r+1}(M^\top M) = \lambda_{T}(2M^\top E) =0$. 

By taking the expectation of both sides, we see,
\[
\mathbb{E}[\lambda_{r+1}(X^\top X)] \geq \left( s(\sqrt{n} - \sqrt{T}) \right)^2
\]
due to Gordon’s theorem, which says that $\mathbb{E}[\sigma_T(E)]\geq s(\sqrt{n}-\sqrt{T})$.  Since $\sigma^*_X = \sqrt{\lambda_{r+1}(X^\top X)}$, we obtain $\mathbb{E}[\sigma^*_X] = \mathbb{E}[\sqrt{\lambda_{r+1}(X^\top X)}] \geq s(\sqrt{n} - \sqrt{T})$.

Combining these two bounds and rearranging terms, we get the desired difference, and see that the lower bound on $\sigma^*_X$ is greater than the upper bound on $\sigma^*_A$ when $n_A < n + 4T -4 \sqrt{nT}$.
\end{proof}

\subsubsection{Sub-Gaussian Noise Setting}

Next, we consider the the sub-gaussian noise setting (Definition \ref{def.subgauss}) where $\|E_{i,t}\|_{\psi_2}=K$. Our result in this setting, Corollary \ref{cor.singval.subgauss}, follows from Theorem \ref{thm.singval.gauss} by instantiating Theorem \ref{thm.vershynin.539} from \cite{vershynin2010introduction}
 instead of Gordon's theorem.

\begin{definition}[Sub-gaussian norm]\label{def.subgauss}
The sub-gaussian norm of $X$, denoted by $||X||_{\psi_2}$ is defined as,
\[
||X||_{\psi_2} = \sup_{p\geq 1} \frac{1}{\sqrt{p}} \left( \mathbb{E}[|X|^{p}] \right)^{1/p}.
\]
\end{definition}

\begin{theorem}[Theorem 5.39 of \cite{vershynin2010introduction}]\label{thm.vershynin.539}
  Let $A$ be an $N \times n$ matrix whose rows $A_i$ are independent
  sub-gaussian isotropic random vectors in $\mathbb{R}^n$.
  Then for every $t \ge 0$, with probability at least $1 - 2\exp(-ct^2)$ one has
  \begin{equation*}	
  \sqrt{N} - C \sqrt{n} - t \le \sigma_{min}(A) \le \sigma_{max}(A) \le \sqrt{N} + C \sqrt{n} + t.
  \end{equation*}
  Here $C = C_K$, $c = c_K > 0$ depend only on the subgaussian norm 
  $K = \max_i \|A_i\|_{\psi_2}$ of the rows.
\end{theorem}

\begin{corollary}[Singular Value Concentration with Sub-gaussian Noise]\label{cor.singval.subgauss}
Let the noise terms satisfy $\|E_{i,t}\|_{\psi_2}=K$. For every $t\geq 0$, if $r<T$ and $n_A < \left(\sqrt{n} - CK^2\sqrt{T} -2t \right)^2$, then with probability at least $1-2 e^{-ct^2}$,
\[\sigma^*_X - \sigma^*_A  \geq \sqrt{n}-\sqrt{n_A} - CK^2\sqrt{T} -2t,\]
where $C>0$ and $c>0$ are constants, and only \(c>0\) depends on the sub-gaussian norm $K = \|E_{i,t}\|_{\psi_2}$.
\end{corollary}

\subsubsection{Heavy-tailed Noise Settings}

Finally, we consider settings where noise comes from a heavy-tailed distribution. This is the most challenging of the three settings considered because the random noise terms will be less concentrated around zero, and thus learning from the noisy data will be more difficult. Using the bound in Theorem \ref{thm.vershynin.541} on maximum and minimum singular values of the noise matrix, we can draw a lower bound on the gap of singular values for heavy-tailed distributions. This is used in place of Theorem \ref{thm.vershynin.539} or Gordon's theorem to prove Corollary \ref{cor.singval.heavy}.

\begin{theorem}[Theorem 5.41 of \cite{vershynin2010introduction}]\label{thm.vershynin.541}
  Let $A$ be an $N \times n$ matrix whose rows $A_i$ are independent
  isotropic random vectors in $\mathbb{R}^n$. Let $m$ be a number such that 
  $\|A_i\|_2 \le \sqrt{m}$ almost surely for all $i$. 
  Then for every $t \ge 0$, one has
  \begin{equation*}
  \sqrt{N} - t \sqrt{m} \le \sigma_{min}(A) \le \sigma_{max}(A) \le \sqrt{N} + t \sqrt{m}
  \end{equation*}
  with probability at least $1 - 2 n \cdot \exp(-ct^2)$, 
  where $c>0$ is an absolute constant.  
\end{theorem}

\begin{corollary}[Singular Value Concentration with Heavy-tail Noise]\label{cor.singval.heavy}
Let the noise terms $E_{i,t}$ follow a heavy-tailed distribution.
If $r<T$ and $n_A < n + 4Tt^2 -4 t\sqrt{nT}$, then for every $t\geq 0$, with probability at least $1-2T e^{-ct^2}$,
\[\mathbb{E}[\sigma^*_X - \sigma^*_A]  \geq \sqrt{n}-\sqrt{n_A} - 2t\sqrt{T}.\]
\end{corollary}

In Section \ref{sec.theory.improve}, the bound of Theorem \ref{thm.singval.gauss} will be used to show improvement in SC performance from using ClusterSC in the presence of Gaussian noise terms; if needed, one could instead adopt Corollary \ref{cor.singval.subgauss} and Corollary \ref{cor.singval.heavy} depending on the relevant assumptions about the noise distributions (sug-Gaussian or heavy-tailed) in a given application.

\subsection{Improvement in SC Performance}\label{sec.theory.improve}

Finally, we translate the effect of subgroup specialization induced by ClusterSC into an improvement in the upper bound of the prediction error. 
We compare the performance of ClusterSC, which uses only the selected donors $A$ as an input for $\mathcal{M}$, against the performance of SC using the whole donor pool $X$, and give results for pre-intervention (Theorem \ref{thm.pre.bound}) and post-intervention error (Theorem \ref{thm.post.bound}).

Let $x_0 = m_0 + e_0$ be the placebo target unit that did not receive the intervention. Then, our goal is to construct an SC prediction $\hat{m}_0$ that approximate $m_0$ as accurately as possible.

Our first main result in this section is Theorem \ref{thm.pre.bound}, which bounds the improvement in pre-intervention MSE from using the selected donor pool $A$ instead of the full donor pool $X$. To do this, Lemma \ref{lem.pre.mse} first gives an upper bound on the pre-intervention MSE of standard synthetic control without clustering (i.e., Algorithm \ref{alg.sc.core}).

\begin{lemma}[Pre-intervention MSE of SC]\label{lem.pre.mse}
Given donor matrix $X\in \mathbb{R}^{n \times T}$, target unit $x_0=m_0+e_0$, rank parameter $r$, noise distribution  $E_{i,t} \sim \mathcal{N}(0, s^2)$, and SC weights $\hat{f} \leftarrow \mathcal{M}(X, x_0^-;r)$ learned using Algorithm \ref{alg.sc.core}, then,
\begin{align*}
    \MSE(\hat{m}_0^-;X)\leq \frac{\mu^2}{T_0}\mathbb{E}[ ( \sigma^*_X + 2s(\sqrt{n} + \sqrt{T})) ^2 ] + \frac{2s^2r}{T_0}.
\end{align*}
\end{lemma}

\begin{proof}
From Lemma \ref{lemma.pre.mse.universal} presented in Appendix \ref{app.useful.3},
\begin{align}\label{eq.lem21.0}
    \mathbb{E}[\|m_0^--\hat{m}_0^-\|^2] \leq \mathbb{E}[\|(M^- - \hat{M}^-)^\top f^*\|^2]+2s^2r.
\end{align}
We bound the first term inside the expectation by
\begin{align}\label{eq.lem21.1}
    \|(M^- - \hat{M}^-)^\top f^*\|^2 \leq \|M^- - \hat{M}^-\|^2 \| f^*\|^2, 
\end{align}
using the property of the operator norm: $\|Ax\| \leq \|A\|\cdot\|x\|$ for any matrix $A$ and vector $x$.
We bound the first term of \eqref{eq.lem21.1} by
\begin{align}
\|M^- - \hat{M}^-\| \leq \|M - \hat{M}\| &\leq \sigma^*_X + 2\|X-M\| \nonumber\\
&\leq \sigma^*_X + 2\|E\|. \label{eq.lem21.2}
\end{align}
Combing these bounds and the assumption $\|f^*\|\leq \mu$, we obtain 
\begin{align*}
    \MSE(\hat{m}_0^-;X) 
    &= \mathbb{E}[\frac{1}{T_0}\| m_0^- - \hat{M}^{-\top}\hat{f}\|^2] \\
    &= \frac{1}{T_0}\mathbb{E}[\| m_0^- - \hat{m}_0^-\|^2] \\
    &\leq \frac{1}{T_0} \mathbb{E}\left[ (\sigma^*_X + 2\|E\|)^2 \right] \mu^2 + \frac{2s^2r}{T_0}. && \mbox{(by Equations \eqref{eq.lem21.0}, \eqref{eq.lem21.1}, and \eqref{eq.lem21.2})}
\end{align*}
Using the fact that $\mathbb{E}[\|E\|] \leq s(\sqrt{n} + \sqrt{T})$ completes the proof.
\end{proof}

Combining Lemma \ref{lem.pre.mse} with the bounds on singular values in Theorem \ref{thm.singval.gauss} allows us to show that the upper bound on pre-intervention MSE decreases when using selected donor pool $A$  instead of the full donor pool $X$ (Theorem \ref{thm.pre.bound}).

\begin{restatable}{theorem}
{thmprebound}\label{thm.pre.bound}
If $n_A < n + 4T -4 \sqrt{nT}$, then the upper bound on pre-intervention MSE of Algorithm \ref{alg.csc} is strictly smaller than that of Algorithm \ref{alg.sc.core}, and the difference in the upper bounds is $\Omega(s^2 n)$.    
\end{restatable}

Next, we analyze the post-intervention root mean squared error (RMSE), and show similar improvements when changing from $X$ to $A$ (Theorem \ref{thm.post.bound}). First, Lemma \ref{lemma.post.rmse} gives an upper bound on the post-intervention error of SC without clustering (Algorithm \ref{alg.sc.core}), under the standard assumption that the SC weights $\hat{f} \leftarrow \mathcal{M}(X, x_0^-;r)$ satisfy $\|\hat{f}\|_2 \leq \eta$ for some $\eta \geq 0$ \citep{rsc}.

\begin{lemma}[Post-intervention RMSE of SC]\label{lemma.post.rmse}
Given a donor matrix $X\in \mathbb{R}^{n \times T}$, a target $x_0$, rank parameter $r$, noise distribution $E_{i,t} \sim \mathcal{N}(0, s^2)$, and SC weights $\hat{f} \leftarrow \mathcal{M}(X, x_0^-;r)$ learned using Algorithm \ref{alg.sc.core},
\[
    \RMSE(\hat{m}_0^+;X) 
    \leq \frac{\eta}{\sqrt{T-T_0}}\mathbb{E}[\sigma^*_X + 2s(\sqrt{n}+\sqrt{T}) ] \sqrt{n}(\mu + \eta).
\]
\end{lemma}

\begin{proof}
We use triangle inequality and the property of induced norm to upper bound the following quantity:
\begin{align*}
    \|m_0^+-\hat{m}_0^+\| &= \|(M^+)^\top f^* - (\hat{M}^+)^\top \hat{f} \| \\
    &\leq \|(M^+ - \hat{M}^+)^\top \hat{f} \| + \|(M^+)^\top (f^* -\hat{f} )\|\\
    &\leq \|M^+ - \hat{M}^+\| \cdot \|\hat{f} \| + \|M^+\|\cdot\|f^* -\hat{f}\|\\
    &\leq \|M^+ - \hat{M}^+\|  \eta + \|M^+\|_F (\|f^*\| + \|\hat{f}\|).
\end{align*}
Taking the expectation of both sides and applying Lemma \ref{lemma:general_threshold}, which shows that $\mathbb{E}[\|M^+ - \hat{M}^+\|] \leq\mathbb{E}[\sigma^*_X + 2\|E\|_2]$, and the assumptions that $\|f^*\|\leq \mu$ and $\|\hat{f}\|\leq \eta$, yields:
\begin{align*}
    \mathbb{E}[\|m_0^+-\hat{m}_0^+\|] &\leq \mathbb{E}[\sigma^*_X + 2\|E\|_2]  \eta + \|M^+\|_F(\mu+\eta).
\end{align*}
Since $\|M^+\|_F \leq \sqrt{n(T-T_0)}$, we obtain, 
\begin{align}\label{eq.lem23.1}
    \frac{1}{\sqrt{T-T_0}}\mathbb{E}[\|m_0^+-\hat{m}_0^+\|] &\leq \frac{\eta}{\sqrt{T-T_0}}\mathbb{E}[\sigma^*_X + 2\|E\|_2]  + \sqrt{n}(\mu+\eta).
\end{align}
Finally, we can bound the post-intervention RMSE as follows.
\begin{align*}
    \RMSE(\hat{m}_0^+;X) 
    &= \mathbb{E}[\frac{1}{\sqrt{T-T_0}}\| m^+ - \hat{M}^{+\top}\hat{f}\|]\\
    &= \frac{1}{\sqrt{T-T_0}}\mathbb{E}[\| m_0^+ - \hat m_0^+\|]\\
    &\leq \frac{\eta}{\sqrt{T-T_0}} \mathbb{E}[\sigma^*_X + 2\|E\|_2]   + \sqrt{n}(\mu+\eta). && \mbox{(by Equation \eqref{eq.lem23.1})} \qedhere
\end{align*}
\end{proof}

The bound on RMSE from Lemma \ref{lemma.post.rmse} can be combined with the bound on the difference in singular values from Theorem \ref{thm.singval.gauss} to analyze the difference in post-intervention error between using $A$ versus $X$ as the donor matrix.
Again, the upper bound stated in Lemma \ref{lemma.post.rmse} has three elements that changes when the donor matrix becomes $A$ instead of $X$: $\sigma^*_X$ to $\sigma^*_A$, $n$ to $n_A$, and $r$ to $r_S$. All three changes reduce the bound, and hence the upper bound on post-intervention error strictly decreases when using Algorithm \ref{alg.csc}.
Theorem \ref{thm.post.bound} gives a lower bound on the improvement of the post-intervention error bound.

\begin{restatable}{theorem}{postbound}\label{thm.post.bound}
If $n_A < n + 4T -4 \sqrt{nT}$, then 
the upper bound on post-intervention RMSE of Algorithm \ref{alg.csc} is strictly smaller than that of Algorithm \ref{alg.sc.core}, and the difference in the upper bounds is  $\Omega(s\sqrt{n})$.
\end{restatable}

\section{Empirical Evaluations}\label{sec.empirical}

In this section, we test various design choices of ClusterSC on simulated datasets (Section \ref{sec.simu}), and demonstrate ClusterSC on a real-world dataset (Section \ref{sec.realworld}).

\subsection{Evaluation on Synthetic Datasets}\label{sec.simu}

To estimate a realistic size for the synthetic dataset, we turn to the literature that has applied SC on disaggregated datasets.
\cite{abadie2021penalized} adopted SC to measure the effect of participation in a government program on an individual's yearly income. They constructed SC instances out of $n=2490$ individuals as a donor pool and with $10$ covariates (equivalent to $T_0$).
\cite{robbins2017framework} examined the effect of a crime intervention on crime levels measured at the census block level. With $3535$ donor units, SC was constructed with $T_0=12$ pre-intervention time-series measurements along with auxiliary variables.
\cite{vagni2021earnings} showed that having a child reduces womens' earnings by constructing SC with $n=630$ women as donors.
$T_0$ varied depending on the woman's first childbirth year, and was at most $7$.

Based on this, we choose $T=10$ and $n \in \{1000, 2000\}$ and set $T_0 = 8$.
We construct a dataset $X$ with two subgroups $A = S + E_A$ and $B = S' + E_B$, with even split ($n_A/n_B =1$).
The signal $S$ (or $S'$) is made by sum of multiple sinusoidal time series.
Let the rank of $S$ be $r_S$. Then, we sample three parameters—$\alpha_i$ (magnitude), $\omega_i$ (frequency), and $\phi_i$ (delay)—to generate a sine wave signal $v_{i,t} = \alpha_i \sin{(2\pi\omega_i t + \phi_i)}, \; \forall i \in [r_S]$. These parameters were independently sampled from the following distributions: $\alpha_i \sim \Beta(2,2)$, $\omega_i \sim \Unif(1,3)$, $\phi_i \sim \mathcal{N}(0,1)$ for $A$, and $\alpha_i \sim \Beta(2,5)$, $\omega_i \sim \Unif(3,6)$, $\phi_i \sim \mathcal{N}(0,1)$ for $B$.
The observation matrix is then constructed with elements $S_{i,t} = \sum_{i=1}^k w_i \cdot v_{i,t}$, where $w_i \sim \Unif([0,1]), \; \forall i \in [r_S]$. 
Finally, we introduce observational noise to yield $A_{i,t} = S_{i,t} + E_{i,t}$, where $E_{i,t} \sim \mathcal{N}(0, s^2)$ for varying levels of $s$ from $0.1$ to $0.4$ with $0.05$ interval.
We repeat the same process for $B$, and concatenate the two matrices to make 
$X=[A^\top,B^\top]^\top$. 
$500$ datasets were generated for the experiment, and the \texttt{sklearn} implementation of Lloyd's $k$-means algorithm with \texttt{k-means++} initialization and silhouette scores\footnote{\url{https://scikit-learn.org/1.5/modules/generated/sklearn.metrics.silhouette_score.html}} to find $k$ were used.

For each dataset, we perform a leave-one-out placebo test on $30\%$ of $A$,
by choosing one unit in $A$ as a target and the rest in $X$ as a potential donor pool.
For each target, we test two methods: (i) ClusterSC, using a subset of the donors $A$ selected via Algorithm \ref{alg.csc}, 
and (ii) SC without clustering, using the whole donor pool $X$.
Our method can flexibly adopt different versions of SC methods, and we present the results using Ridge regression in this section. 
We provide additional empirical evaluations with different choice of regression methods (OLS, Ridge, and Lasso) in Appendix \ref{app.simulation}.

Figure \ref{fig.box} shows the distribution of median MSE for the two algorithms, when $n_A=n_B=500$. 
The boxplot shows the quartiles of MSE, the whiskers extend to the furthest datapoint within $1.5$ times the interquartile range, and the rest are shown as small dots.
We observe that ClusterSC consistently outperforms SC, across all noise levels.
This aligns with our Theorem \ref{thm.post.bound}, which promises a tighter error bound.

\begin{figure}[h]
\CommonHeightRow{%
    \begin{floatrow}[2]%
        \ffigbox[\FBwidth]
        {\includegraphics[height=\CommonHeight]{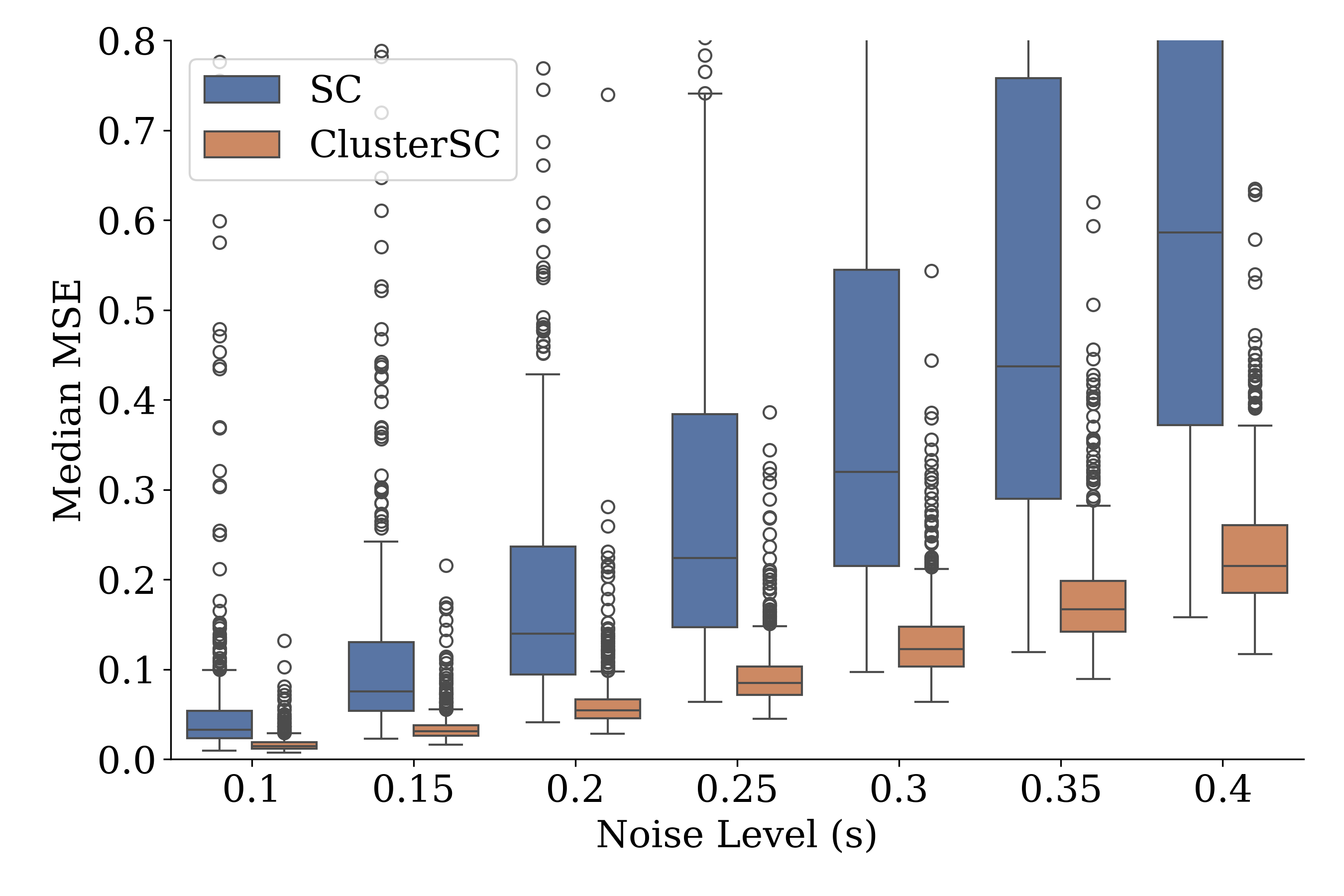}}
        {\caption{Median Post-intervention MSE using the classical SC without our clustering step (blue) and ClusterSC (orange) for varying levels of noise.
}\label{fig.box}}
        \ffigbox[\FBwidth]
        {\includegraphics[height=\CommonHeight]{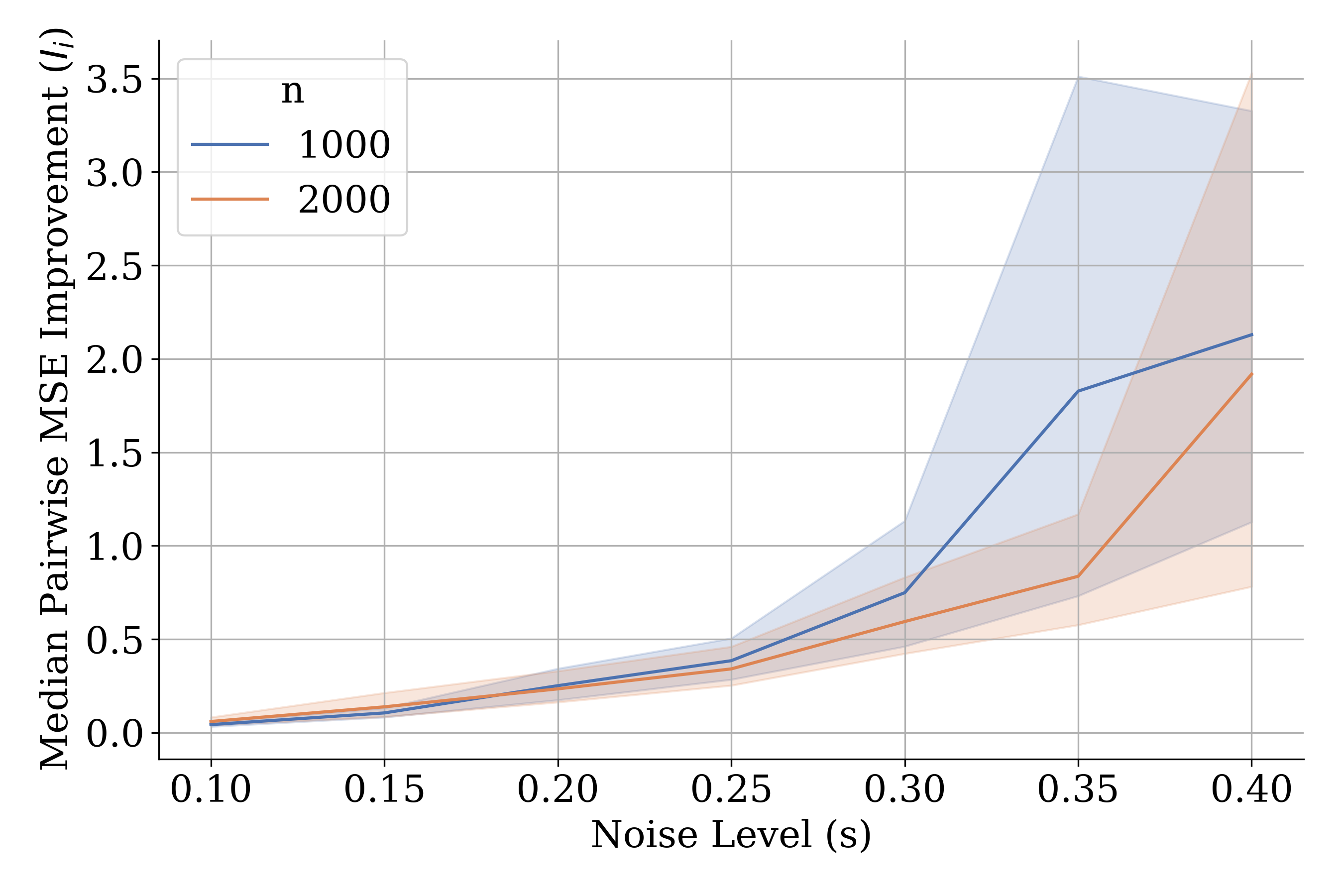}}
        {\caption{Median of the pairwise improvement $I_i$, measured for each dataset, for different noise levels ($s$). Shades represent 95\% confidence interval.
}\label{fig.noise}}
    \end{floatrow}}%
\end{figure}

Next, we define the pairwise improvement for a target $i$ as the difference in post-intervention MSE scores between SC and ClusterSC: $I_i= \MSE(\hat{m}_i^+;X)-\MSE(\hat{m}_i^+;A)$.
Then, we take $\text{median}(I_i)$ as a metric to assess the overall improvement measured from the SC instances constructed from one dataset (under a leave-one-out placebo test).
Figure \ref{fig.noise} shows the median pairwise improvement, $\text{median}(I_i)$, induced by ClusterSC at varying noise level. 
We observe that the median improvement is almost always positive, meaning that more than half of the individuals benefit from using ClusterSC instead of classical SC without clustering.
The improvement grows as noise increases, aligning with our Theorem \ref{thm.post.bound}.
We provide additional empirical evaluations with different choice of regression methods (OLS, Ridge, and Lasso) in Appendix \ref{app.simulation}.

\subsection{Evaluation on Real-world Dataset}
\label{sec.realworld}

Next, we evaluate ClusterSC using housing price index (HPI) data from the U.S. Federal Housing Finance Agency.\footnote{https://www.fhfa.gov/data/hpi/datasets} To avoid the effects of the subprime mortgage crisis ($2007-2010$), we use ten years of quarterly HPI data from $1997$ to $2006$, yielding a total of $T=40$ time points. 
The dataset was preprocessed to retain only metropolitan areas without missing data, resulting in $n=400$ units.

To evaluate ClusterSC, we conduct a placebo test to assess the model's ability to accurately predict \emph{observations} that would serve as \emph{counterfactual} outcomes in the presence of an intervention.
We formulate $100$ iterations of test, each involving a random split of the units into a donor set ($80\%$) and a target set ($20\%$). In each iteration, a unique SC model is fitted for each target using the corresponding donor set. The accuracy of an iteration is measured by the median MSE of the post-intervention predictions.

\begin{wrapfigure}{r}{0.5\textwidth}
  \begin{center}
    \includegraphics[width=\textwidth]{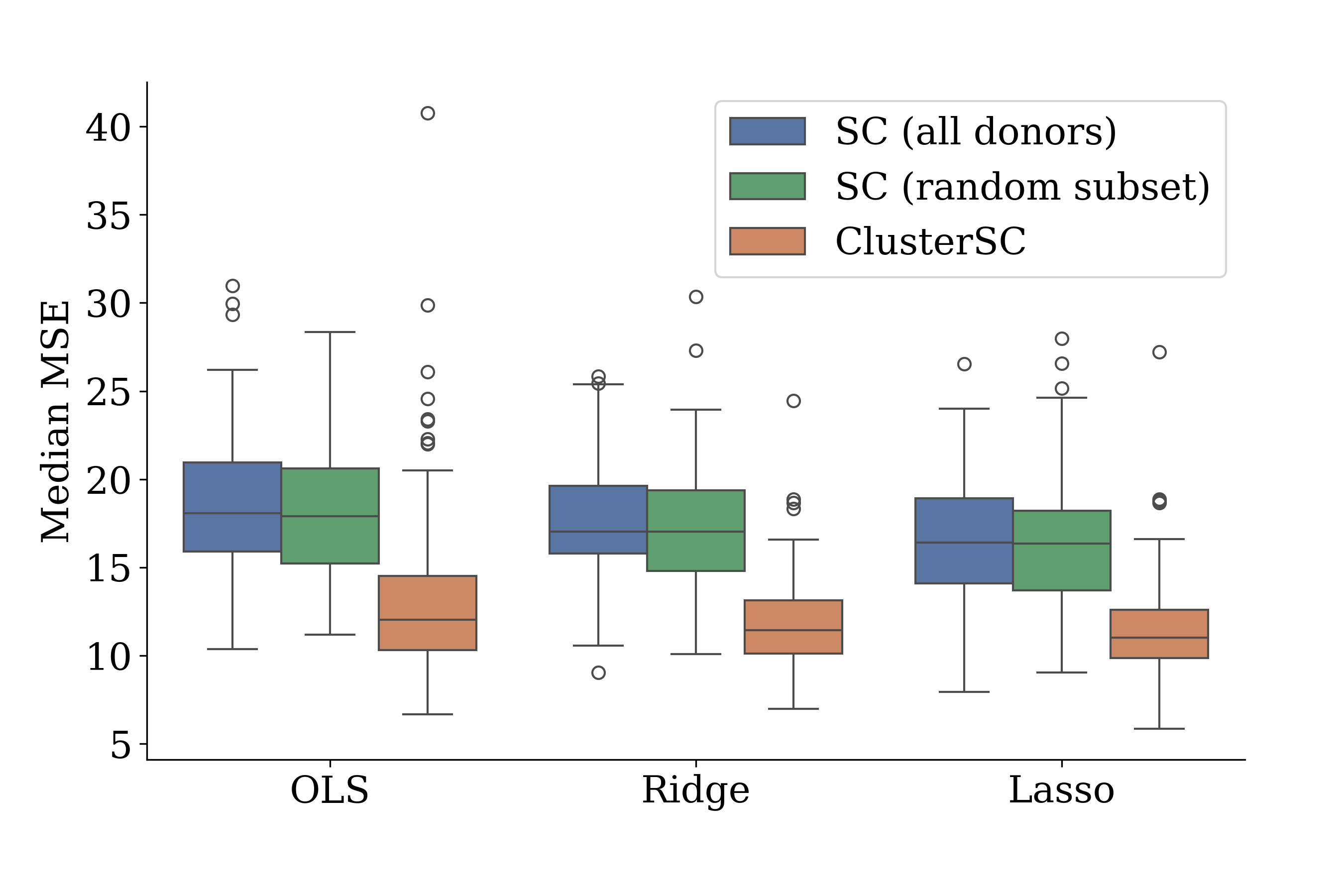}
  \end{center}
    \caption{Comparison of ClusterSC and two SC benchmarks on different regression methods (OLS, Ridge, and Lasso). Each boxplot contains 100 points representing the median MSE of each iteration.
    }
    \label{fig.hpi_results}
\end{wrapfigure}

Figure \ref{fig.hpi_results} presents boxplots of the median MSE measured over $100$ iterations.
We fix the post-intervention period to year $2006$ (four quarterly data points) and use nine years prior to $2006$ as pre-intervention time points ($T_0=36$).
The performance of ClusterSC (orange) is compared against two benchmarks. The first benchmark applies SC using the entire donor pool (blue), while the second benchmark selects a randomly subsampled donor set of the same size as the ClusterSC-selected donor set (green). For example, if ClusterSC selects a cluster of $50$ donors for a given target, the second benchmark randomly selects $50$ donors from the full donor pool.
The plot consists of nine boxplots with different choice of learning method for SC weights (e.g., Step 3 of Algorithm \ref{alg.sc.core}):
OLS (first three), Ridge (middle three), and Lasso (last three).
Regularization coefficients for Ridge and Lasso were set to $0.1$ after testing values of $0.01$, $0.05$, $0.1$ and $0.2$, which resulted in minimal performance differences.
For all SC instances, we used $k=2$ clusters\footnote{Based on silhouette scores.} and determined the rank cutoff for HSVT at the 95\% threshold.

Compared to the first benchmark, SC with all donors (blue), the second benchmark, SC with a random subset of donors (green), does not show any meaningful change in Figure \ref{fig.hpi_results}.
However, ClusterSC (orange) consistently achieves lower median MSE compared to both benchmarks over all learning methods (OLS, Ridge, or Lasso).
This indicates that the clustering approach in ClusterSC improves prediction accuracy by selecting a more relevant donor pool, not just by using fewer donors.
The thinner orange boxes, which indicate lower variance in MSE, also suggest that the clustering approach not only improves accuracy but also enhances stability.

\section{Discussion and Future Work}

This paper presents a novel approach to synthetic control (SC) on disaggregate-level datasets, addressing the challenges of higher noise and increased dimensionality by incorporating a clustering step. To the best of our knowledge, this is the first method to directly reduce the dimension of regression weights, in contrast to approaches that rely on regularization to suppress the number of active donors \citep{abadie2021penalized,chernozhukov2021exact, rsc, doudchenko2016balancing}.
Our ClusterSC advances synthetic control methodology to be better suited for applications where individual-level conditional treatment effects are of interest, such as in drug trials or targeted marketing analyses.

ClusterSC is supported by two main theoretical guarantees. Theorem \ref{thm.cluster_consistency} demonstrates the accuracy of our clustering step in identifying intrinsic cluster structure among the donor latent variables $\Theta$.
Theorems \ref{thm.pre.bound} and \ref{thm.post.bound} establish a tighter upper bound on prediction error induced by our algorithm, which is empirically validated in Section \ref{sec.simu} with simulation data and Section \ref{sec.realworld} with a housing price index dataset. In both experiments, ClusterSC consistently shows significant improvement across the choice of the learning algorithm for SC weights (e.g., OLS, Ridge, and Lasso).

The improved signal detection in ClusterSC, achieved by partitioning the donors, relies on the assumption that each \emph{cluster} exhibits a more pronounced low-rank structure than the combined matrix. This aligns with prior findings suggesting that, in models such as Gaussian mixtures, the middle components of singular value decomposition can carry more informative signals than the principal component \citep{nadakuditi2013most}. By incorporating a clustering step, ClusterSC effectively isolates these mixtures, ensuring that the principal components remain the most informative.
A similar approach has been explored in matrix completion, where rows are iteratively partitioned based on their projections onto the principal component \citep{ruchansky2017targeted}. In the same spirit, ClusterSC leverages the top few principal components to identify meaningful clusters.

Conceptually, ClusterSC shares similarities with Lasso in that it selects a small subset of donors. While Lasso provides effective regularization when the donor pool is large, ClusterSC has demonstrated further refinement in the experiments presented in Appendix \ref{sec.lasso} and Section \ref{sec.realworld}.  
The impact of the clustering step on different variations of synthetic control requires further analysis. However, the fundamental principle of concentrating on the most informative singular values remains valid under the common assumption that the underlying data follows a latent variable model, resulting in an approximately low-rank matrix.

Lastly, we acknowledge a potential fairness issue in our approach. As shown empirically in Section \ref{sec.empirical}, our method guarantees improved overall performance of the SC algorithm. However, it does not ensure that the prediction error will decrease for every individual target unit---while the majority of units may benefit, some could experience worse outcomes. This uneven distribution of benefits raises concerns about fairness, especially in individual-level datasets.
Investigating the potential disproportionate effects on minority groups presents an avenue for future research.

\bibliography{references}

\input{arxiv_appendix}

\end{document}

%% file: arxiv_appendix.tex
\appendix
\onecolumn

\section{Technical Definitions}

Throughout the paper, we use lower-case letters to denote a vector $x$ and upper-case letters to denote a matrix $X$. The norms $\|X\|$ and $\|x\|$ refer to the spectral norm and $\ell_2$ norm, respectively.

In this section, we summarize important definitions used in our paper. (Some are repeated in the main part too.)

\begin{definition}[Sub-gaussian norm]
The sub-gaussian norm of $X$, denoted by $||X||_{\psi_2}$ is defined as
\[
||X||_{\psi_2} = \sup_{p\geq 1} \frac{1}{\sqrt{p}} \left( \mathbb{E}[|X|^{p}] \right)^{1/p}.
\]
\end{definition}

\begin{definition}[Bilipschitz continuity]
Let \((X,d), (Y,\rho)\) be metric spaces. A map \(g:(X,d) \mapsto (Y,d')\) is \(L\)-bilipschitz, for \(L > 0\), if, for all \(x,x'\in X\)
\[
\frac{1}{L} \; d(x,x') \leq \rho(g(x),g(x')) \leq L \;  d(x,x')
\]
\end{definition}

\section{Useful Theorems and Lemmas from Prior Work}\label{app.useful}

\subsection{Related to Theorem \ref{thm.cluster_consistency}}

\begin{theorem}[Theorem 5.1 (i) of \cite{ostrovsky2013effectiveness}]\label{thm.ostrovsky}
Suppose that $X$ is $\varepsilon$-separated with $k$ clusters. If there is a Voronoi partition $P=\{P_1, \ldots, P_k\}$ such that 
\[
\Delta_k^2(X; P) \leq \alpha \Delta_{k-1}^2(X)
\]
for some $\alpha \in (0, \frac{1-401\varepsilon^2}{400}]$,
then for each cluster $P_i$, there is a cluster $P'_i$ induced by a distinct optimal center, such that:
\[|P_i \ominus P'_i| \leq 161\varepsilon^2 |P'_i|\]
where $A \ominus B$ denotes the symmetric difference between sets $A$ and $B$.
\end{theorem}

\begin{theorem}[Theorem 5.1 (ii) of \cite{ostrovsky2013effectiveness}]\label{thm.ostrovsky.ii}
Let \(X= \{x_1,...,x_n\} \subset \mathbb{R}^d\) be \(\varepsilon\)-separated and let \(X' = \{x_1',...,x_n'\}\) such that \(\|x_i-x_i'\|\leq \frac{\epsilon \Delta_{k-1}(X)}{\sqrt{n}}\). Then \(\Delta_{k}^2(X') \leq \frac{8\varepsilon^2}{1-2\varepsilon^2} \Delta_{k-1}^2(X')\).
\end{theorem}

\begin{theorem}[Claim 6 of \cite{kumar2010clustering}]\label{thm.kk.claim}
    If \(X\) is $\varepsilon$-separated with $k$ clusters, then all but $\varepsilon^2$ fraction of points in $X$ satisfy the proximity condition.
\end{theorem}

\begin{theorem}[Theorem 2.2 of \cite{kumar2010clustering}]\label{thm.kk.theorem}
    If all but $\varepsilon$ fraction of points satisfy the proximity condition,
    then there exists an algorithm running in polynomial time which correctly partitions all but \(O(k^2\varepsilon n)\) points.
\end{theorem}

\subsection{Related to Theorem \ref{thm.singval.gauss}}
First of all, we introduce two version sof Weyl's inequality used in the proof of Theorem \ref{thm.singval.gauss}.

\begin{theorem}[Weyl's Inequality on Singular Values]\label{thm.weyl.sing}
For matrices $A$ and $B$ in $\mathbb{R}^{n \times m}$, let $k = \min(n, m)$. Then, the following holds for all $i, j \in [k]$, $i+j-1\leq k$.
\[
\sigma_{i+j-1}(A+B) \leq \sigma_i(A) + \sigma_j(B)
\]
\end{theorem}

\begin{theorem}[Weyl's Inequality on Eigenvalues]\label{thm.weyl.eigen}
For square matrices $A$ and $B$ in $\mathbb{R}^{n \times n}$, the following holds for all $i, j \in [n]$, $i+j-1\leq k$
\[
\lambda_{i+j-1}(A+B) \leq \lambda_i(A) + \lambda_j(B)
\]
And for all $i \in [n]$,
\[
 \lambda_i(A) + \lambda_n(B) \leq \lambda_{i}(A+B) \leq \lambda_i(A) + \lambda_1(B).
\]
\end{theorem}

Next, we introduce Gordon's theorem that bounds the singular values of Gaussian matrices, also used in the proof of Theorem \ref{thm.singval.gauss}.

\begin{theorem}[Gordon's theorem for Gaussian matrices] \index{Gordon's theorem} \label{Gaussian}
  Let $A$ be an $N \times n$ matrix whose entries
  are independent standard normal random variables. Then
  $$
  \sqrt{N} - \sqrt{n} \le \E[\sigma_{min}(A)] \le \E [\sigma_{max}(A)] \le \sqrt{N} + \sqrt{n}.
  $$
\end{theorem}

\subsection{Related to Theorem \ref{thm.pre.bound} and Theorem \ref{thm.post.bound}}\label{app.useful.3}

In this section, we introduce auxiliary results from the literature that will be used in the proofs of Theorem \ref{thm.pre.bound} and Theorem \ref{thm.post.bound}. The first such result is from \cite{chatterjee2015matrix}.

\begin{theorem}[Perturbation of Singular Values, \cite{chatterjee2015matrix}]\label{thm.singval.perturb}
Let $A$ and $B$ be two $m \times n$ matrices. Let $k = \min\{m,n\}$. Let $\sigma_1(A),\dots, \sigma_k(A)$ be the singular values of $A$ in decreasing order and repeated by multiplicities. Similarly, we define $\sigma_1(B),\dots, \sigma_k(B)$ for $B$ and $\sigma_1(A-B),\dots, \sigma_k(A-B)$ for matrix $A-B$. Then,
\begin{align*}
\max_{1 \le i \le k} | \sigma_i(A) - \sigma_i(B) |&\le \max_{1 \le i \le k} |\sigma_i(A-B)|.
\end{align*}
\end{theorem}

Using Theorem \ref{thm.singval.perturb}, we derive the following lemma. We provide a proof for completeness, but the proof is also presented in \cite{chatterjee2015matrix} and \cite{rsc}.

\begin{lemma}[Approximation Bound Between Two Matrices, Lemma 20 of \cite{rsc}.]\label{lemma:general_threshold}
Let $A$ and $B$ be two matrices of the same size.
Let $A = \sum_{i=1}^{m} \sigma_i(A) u_i v_i^\top$ be the singular value decomposition of $A$ with $\sigma_1(A), \dots, \sigma_m(A)$ in decreasing order and with repeated multiplicities. For any choice of $\mu \geq 0$, let $S = \{i : \sigma_i \geq \mu \}$. Then, define
\begin{align*}
    \hat{B} &= \sum_{i \in S} \sigma_i(A) u_i v_i^\top.
\end{align*}
Let $\sigma_i(B)$ be the singular values of $B$ in decreasing order and repeated by multiplicities, with $\sigma^*_B = \max_{i \notin S} \sigma_i(B)$. Then
\begin{align*}
    ||\hat{B} - B|| &\leq \sigma^*_B + 2||A - B||. 
\end{align*}
\end{lemma}

\begin{proof}
By Theorem \ref{thm.singval.perturb}, we know that $\sigma_i(A) \le \sigma_i(B) + ||A - B||$ for all $i$. First applying the triangle inequality and then using this fact gives:
\begin{align*}
    \norm{ \hat{B} - B} &\le \norm{\hat{B} - A} + \norm{A - B}
    \\ &= \max_{i \notin S} \sigma_i(A) + \norm{A - B}
    \\ &\le \max_{i \notin S} \Big(\sigma_i(B) + \norm{A - B} \Big) + \norm{A - B}
    \\ &= \sigma_B^* + 2 \norm{A - B}. 
\end{align*} 
\end{proof}

The following lemma comes from \cite{rsc}. We provide a simplified version of the proof here using our notation for completeness.

\begin{lemma}[Universal Bound on Pre-intervention MSE of OLS, Lemma 25 of \cite{rsc}.]\label{lemma.pre.mse.universal}
Suppose $x_0^- = m_0^- + \epsilon_0^-$ with $\mathbb{E}[\epsilon_{0,j}] = 0$ and $\text{Var}(\epsilon_{0,j}) \le s^2$ for all $j \in [T_0]$. Let $f^*$ be the true weights assumed in Section \ref{sec.model} and $\hat{f}$ be the output of Algorithm \ref{alg.sc.core}. Then, 
\begin{equation} \label{eq:x.20}
    \mathbb{E}\norm{ m_0^- - \hat{m}_0^-}^2 \leq \mathbb{E} \norm{(M^- - \hat{M}^-)^\top  f^*}^2  + 2s^2 r,
\end{equation}
where $r = rank(M)$.
\end{lemma}

\begin{proof}
For easier notation, we define the following:
\begin{align*}
	Q := (M^{-})^\top, \;\; \hat{Q} := (\hat{M}^{-})^\top.
\end{align*}
Then, the following is true:
\begin{align*}
	m_0^- := Q f^*, \;\; x_0^- := \hat{Q} \hat{f}.
\end{align*}
Let $\dagger$ to denote a pseudoinverse, then,
\begin{align*}
	f^*= Q^\dagger m_0^-, \;\; \hat{f}= \hat{Q}^\dagger x_0^- .
\end{align*}

Recall that the target pre-intervention data decomposes into: $x_0^{-} = m_0^- + \epsilon_0^{-}$
and $m_0^- = Q f^*$. 
Since $\hat{f}$ minimizes $\norm{x_0^{-} - \hat{Q} f}$ over all $f \in \mathbb{R}^{n}$, then:
\begin{align*}
	\norm{m_0^- - \hat{m}_0^-}^2 &= \norm{ (x_0^{-} - \epsilon_0^{-}) - \hat{Q} \hat{f}}^2
	\\ &= \norm{(x_0^{-} - \hat{Q}\hat{f}) + (- \epsilon_0^{-}) }^2
	\\ &= \norm{x_0^{-} - \hat{Q} \hat{f}}^2 + \norm{ \epsilon_0^{-}}^2 + 2 \langle -\epsilon_0^{-}, x_0^{-}- \hat{Q} \hat{f} \rangle
	\\ &\le \norm{x_0^{-} - \hat{Q} f^*}^2 + \norm{ \epsilon_0^{-}}^2 + 2 \langle -\epsilon_0^{-}, x_0^{-}- \hat{Q} \hat{f} \rangle
	\\ &= \norm{(Qf^* + \epsilon_0^{-}) - \hat{Q} f^*}^2 + \norm{ \epsilon_0^{-}}^2 + 2 \langle -\epsilon_0^{-}, x_0^{-}- \hat{Q} \hat{f} \rangle
	\\ &= \norm{(Q - \hat{Q}) f^* + \epsilon_0^{-}}^2 + \norm{ \epsilon_0^{-}}^2 + 2 \langle -\epsilon_0^{-}, x_0^{-}- \hat{Q} \hat{f} \rangle
	\\ &= \norm{(Q - \hat{Q})f^*}^2 + 2\norm{ \epsilon_0^{-}}^2 +   2 \langle \epsilon_0^{-}, (Q - \hat{Q})f^* \rangle + 
	 2 \langle -\epsilon_0^{-}, x_0^{-}- \hat{Q} \hat{f} \rangle.
\end{align*}

By taking expectations, we have,
\begin{align}\label{eq:mse}
	\mathbb{E}\norm{ \hat{m}_0^- - m_0^-}^2  & \le \mathbb{E} \norm{(Q - \hat{Q}) f^*}^2 + 2\mathbb{E}\norm{ \epsilon_0^{-}}^2 + 2\mathbb{E}[ \langle \epsilon_0^{-}, (Q - \hat{Q})f^* \rangle]   + 2\mathbb{E}[ \langle -\epsilon_0^{-}, x_0^{-}- \hat{Q} \hat{f} \rangle].
\end{align}

We next bound the two inner products on the right hand side of Equation \eqref{eq:mse}:
	\begin{align*}
	\mathbb{E}[ \langle \epsilon_0^{-}, (Q - \hat{Q})f^* \rangle] &=  \mathbb{E}[(\epsilon_0^{-})^\top] Q f^* - \mathbb{E}[ (\epsilon_0^{-})^\top \hat{Q}]f^*
	\\ &=  -\mathbb{E}[( \epsilon_0^{-})^\top] \mathbb{E}[\hat{Q}]f^*
	\\ &= 0. && \mbox{(because $\mathbb{E}[ (\epsilon_0^{-})^\top=0$)}
	\end{align*}

For the second inner product, using $\hat{f} = \hat{Q}^{\dagger} x_0^{-}$, 
	\begin{align}
	 \mathbb{E}[\langle -\epsilon_0^{-}, x_0^{-} - \hat{Q}\hat{f} \rangle ] 
	& \quad = \mathbb{E}[(\epsilon_0^{-})^\top \hat{Q} \hat{f}] -\mathbb{E}[(\epsilon_0^{-})^\top x_0^{-}] \notag
	\\ & \quad =  \mathbb{E}[(\epsilon_0^{-})^\top \hat{Q} \hat{Q}^{\dagger} x_0^{-}] -\mathbb{E}[(\epsilon_0^{-})^\top]m_0^- - \mathbb{E}[(\epsilon_0^{-})^\top \epsilon_0^{-}] \notag
	\\ & \quad  = \mathbb{E}[(\epsilon_0^{-})^\top \hat{Q} \hat{Q}^{\dagger}]m_0^- + \mathbb{E}[(\epsilon_0^{-})^\top \hat{Q} \hat{Q}^{\dagger}\epsilon_0^{-}]  - \mathbb{E}[(\epsilon_0^{-})^\top \epsilon_0^{-}] \notag
	\\ &\quad  \stackrel{(a)}{=} \mathbb{E}[(\epsilon_0^{-})^\top] \mathbb{E}[\hat{Q} \hat{Q}^{\dagger}]m_0^- + \mathbb{E}[(\epsilon_0^{-})^\top \hat{Q} \hat{Q}^{\dagger}\epsilon_0^{-}]  - \mathbb{E}[(\epsilon_0^{-})^\top \epsilon_0^{-}] \notag
	\\ &\quad  = \mathbb{E}[(\epsilon_0^{-})^\top \hat{Q} \hat{Q}^{\dagger}\epsilon_0^{-}] - \mathbb{E}\norm{ \epsilon_0^{-}}^2, \label{eq.expectation}
\end{align}
where $(a)$ follows from the independence of noise.

We next bound the remaining term in Equation \eqref{eq.expectation}:
\begin{align*}
\mathbb{E}[(\epsilon_0^{-})^\top \hat{Q} \hat{Q}^{\dagger}\epsilon_0^{-}] &=  \mathbb{E}[ \text{tr}((\epsilon_0^{-})^\top \hat{Q} \hat{Q}^{\dagger}\epsilon_0^{-})]
	\\ &=  \mathbb{E}[ \text{tr}(\hat{Q} \hat{Q}^{\dagger}\epsilon_0^{-}(\epsilon_0^{-})^\top )]
	\\ &=  \text{tr}\Big(\mathbb{E}[ \hat{Q} \hat{Q}^{\dagger}\epsilon_0^{-}(\epsilon_0^{-})^\top ]\Big)
	\\ &= \text{tr}\Big(\mathbb{E}[ \hat{Q} \hat{Q}^{\dagger}] \mathbb{E}[\epsilon_0^{-}(\epsilon_0^{-})^\top ]\Big)
	\\ &\le  \text{tr}\Big(\mathbb{E}[ \hat{Q} \hat{Q}^{\dagger}] s^2 I \Big)
	\\ &= s^2 \mathbb{E}[ \text{tr}(\hat{Q} \hat{Q}^{\dagger}) ]
	\\ &\stackrel{(b)}{=} s^2 \mathbb{E}[\text{rank}(\hat{Q})]
	\\ &\le s^2 r,
\end{align*}
where $(b)$ follows from the fact that $\hat{Q} \hat{Q}^{\dagger}$ is a projection matrix with rank $r$.

Plugging this bound back into Equation \eqref{eq.expectation} gives:
\[
\mathbb{E}[\langle -\epsilon_0^{-}, x_0^{-} - \hat{Q}\hat{f} \rangle ]  \leq s^2 r - \mathbb{E}\norm{ \epsilon_0^{-}}^2.
\]

Finally, we can combine all these bounds back into Equation \eqref{eq:mse} and plug in the expressions for $Q$ and $\hat{Q}$ to get the desired bound:
	\begin{align*}
	\mathbb{E}\norm{ \hat{m}_0^- - m_0^-}^2 &\le \mathbb{E} \norm{(Q - \hat{Q}) f^*}^2 + 2\mathbb{E} \norm{\epsilon_0^{-}}^2 + 2(s^2 r - \mathbb{E} \norm{ \epsilon_0^{-}}^2)
	\\ &=  \mathbb{E} \norm{(Q - \hat{Q}) f^*}^2  + 2s^2 r
    \\ &= \mathbb{E} \norm{((M^{-})^\top - (\hat{M}^{-})^\top) f^*}^2  + 2s^2 r.
	\end{align*}
\end{proof}

\section{Omitted proofs from Section \ref{sec.theory.subgroup}}\label{app.proofs}

In addition to the $k$-means objective defined in Section \ref{sec.subgroup}, 
 we additionally define more variations of the $k$-means objective, that will be used in the analysis.
For a set of points $A={a_1, \ldots, a_n}$, we define the $k$-mean optimal cluster centers $C^A = \{c^A_i\}_{i=1}^{k}$ and the induced Voronoi partition $P^A = \{P^A_i\}_{i=1}^{k}$.
Note that the optimal $k$-means objective can be defined in two equivalent ways,
\[
\Delta_k^2(A)
= \sum_{i\in[n]} \min_{j\in[k]} ||a_i-c^A_j||^2 = \sum_{l\in[k]} \frac{1}{2|P^A_l|}\sum_{i,j \in P^A_l} ||a_i-a_j||^2.
\]
In the course of the proofs, we often use non-optimal $k$-means cost by introducing artificial partitions $\hat{P}$ and cluster centers $\{\hat{c}_i\}_{i=1}^{k}$.
When the new cluster centers are the mean of all points belonging to each cluster, we omit the center and denote,
\begin{align*}
\Delta_k^2(A; \hat{P})
&= \sum_{l\in[k]} \frac{1}{2|\hat{P}_l|}\sum_{i,j \in \hat{P}_l} ||a_i-a_j||^2.
\end{align*}
Hence, $\Delta_k^2(A ; P^A) = \Delta_k^2(A)$ by definition.
When the new cluster centers are not the mean of points in each cluster, we explicitly denote,
\begin{align*}
\Delta_k^2(A; \hat{P}, \{\hat{c}_i\}_{i=1}^{k})&= \sum_{j\in[k]} \sum_{i\in \hat{P}_j} ||a_i-\hat{c}_j||^2.
\end{align*}

\subsection{
Proof of Lemma \ref{lem.bilipseffect}
}

\bilipseffect*
\begin{proof}
We begin with the optimal $k$-means objective in $\Theta$:  
\begin{align}
\Delta_k^2(\Theta) = \Delta_k^2(\Theta; P^\Theta) & \leq \Delta_k^2(\Theta;P^M)  \label{eq.kmeans.bounding.1}\\
&= \sum_{l=1}^k \frac{1}{2|P^M_l|}\sum_{i, j \in P^M_l} \|\theta_i - \theta_j\|^2 \nonumber \\
&\leq L^2 \sum_{l=1}^k \frac{1}{2|P^M_l|}\sum_{i, j \in P^M_l} ||g(\theta_i) - g(\theta_j)||^2 \nonumber \\
&= L^2\Delta_k^2(g(\Theta);P^M) = L^2\Delta_k^2(g(\Theta))\label{eq.kmeans.bounding.2}\\
&\leq L^2\Delta_k^2(g(\Theta); P^\Theta) = L^2 \sum_{l=1}^k \frac{1}{2|P^\Theta_l|}\sum_{i, j \in P^\Theta_l} ||g(\theta_i) - g(\theta_j)||^2 \label{eq.kmeans.bounding.intermediate} \\
&\leq L^4 \sum_{l=1}^k \frac{1}{2|P^\Theta_l|}\sum_{i, j \in P^\Theta_l} \|\theta_i-\theta_j\|^2 = L^4 \Delta_k^2(\Theta ; P^\Theta) =L^4 \Delta_k^2(\Theta)
\label{eq.kmeans.bounding.3}
\end{align}
The first step is because of the optimality of $P$, the second step is from the definition of $\Delta_k^2(\Theta;P^M)$, the third step is because $g$ is $L$-bilipschitz, the fourth step is by definition,
the fifth step is because $P^M$ is optimal for $g(\Theta)$, 
and the fifth step is again because $g$ is $L$-bilipschitz.

Combining Equations \eqref{eq.kmeans.bounding.1}, \eqref{eq.kmeans.bounding.2}, and \eqref{eq.kmeans.bounding.3} and dividing by $L^2$ yields:
\[
\frac{1}{L^2}\Delta^2_k(\Theta) \leq \Delta^2_k(g(\Theta)) \leq L^2 \Delta^2_k(\Theta).
\]
\end{proof}


\subsection{Proof of Lemma \ref{lem.p.prime.1}}

\pprime*

\begin{proof} The proof of this lemma first expands the definition of \(c_i'\) and then rearranges terms:  
    \begin{align*}
       \|g(c_i^\Theta) - c'_i\| &=  
       \Big\|g(c_i^\Theta) - \frac{1}{|P_i^\Theta|}\sum_{j\in P_i^\Theta} g(\theta_j)\Big\|  \\
       &= \frac{1}{|P_i^\Theta|} \Big\|\sum_{j\in P_i^\Theta} (g(c_i^\Theta) - g(\theta_j))\Big\| \\
       &\leq \frac{1}{|P_i^\Theta|} \sum_{j\in P_i^\Theta} \|g(c_i^\Theta) - g(\theta_j)\| && \text{(triangle inequality)}\\
   &\leq \frac{L}{|P_i^\Theta|} \sum_{j\in P_i^\Theta} \|c_i^\Theta - \theta_j\| && \text{(Lipschitzness of \(g\))}  \\
 &\leq L \sqrt{\frac{1}{|P_i^\Theta|} \sum_{j\in P_i^\Theta} \|c_i^\Theta - \theta_j\|^2} = L\cdot r_i(\Theta) && \text{(Jensen's inequality)} 
    \end{align*}
\end{proof}


\subsection{Proof of Lemma \ref{lem.core.set}}

\coreset*

\begin{proof}
Define $d_l = \|\theta_l - c_i^\Theta\|^2$ for alll $l\in P_i^\Theta$, and let $Y$ be a random variable with probability mass distributed uniformly over $l\in P_i^\Theta$.
 Then, 
 \[\mathbb{E}_{l\sim Y}[d_l] = \frac{1}{|P_i^\Theta|} \sum_{l\in P_i^\Theta} \|\theta_l - c_i^\Theta\|^2 = r_i^2(\Theta),\]
 by construction.
 Using Markov's inequality, for all \(t > 0\), we have
 \[\mathbb{P}_{l\sim Y}(d_l \geq t) \leq \frac{\mathbb{E}_{l \sim Y}[d_l]}{t} 
 = \frac{r_i^2(\Theta)}{t} 
 \leq \frac{\frac{\epsilon^2}{1-\epsilon^2} \min_{j\neq i} \|c_i^\Theta - c_j^\Theta\|^2}{t},
 \]
 where the last step comes from Lemma \ref{lem.ri.bound}.
 Take $t = \frac{r_i^2(\Theta)}{\epsilon}= \frac{\epsilon}{1-\epsilon^2} \min_{j\neq i} \|c_i^\Theta - c_j^\Theta\|^2 $.
 Then,
 \[
 \mathbb{P}_{l\sim Y}\left(\|\theta_l - c_i^\Theta\|^2 \geq \frac{r_i^2(\Theta)}{\epsilon}\right) 
 = \mathbb{P}_{l\sim Y}\left(\|\theta_l - c_i^\Theta\| \geq \frac{r_i(\Theta)}{\sqrt{\epsilon}}\right) 
 =  \mathbb{P}_{l\sim Y}\left(\|\theta_l - c_i^\Theta\| \geq \sqrt{\frac{\epsilon}{1-\epsilon^2} }\min_{j\neq i} \|c_i^\Theta - c_j^\Theta\|\}\right) 
 \leq \epsilon
 \]
Hence, there are at most an $\varepsilon$-fraction of points $l$ in $P^\Theta_i$ that will not belong to $core(P^\Theta_i)$.
Thus $|core(P^\Theta_i)|\geq (1- \varepsilon) |P^\Theta_i|$.
\end{proof}


\subsection{Proof of Lemma \ref{lem.p.prime.2}}

\pprimeagain*

\begin{proof}
    Fix \(i\in [k]\). For all \(j \in [k]\setminus \{i\}\) and all \(l\in core(P_i^\Theta) \), by the triangle inequality,
\begin{equation}
    \|c_j^\Theta - \theta_l \| + \|\theta_l - c_i^\Theta\| \geq \|c_j^\Theta - c_i^\Theta\|.
\end{equation} 
Subtracting $2\|\theta_l - c_i^\Theta\|$ from each side gives,
\begin{align*}
    \|c_j^\Theta - \theta_l \| - \|\theta_l - c_i^\Theta\| &\geq \|c_j^\Theta - c_i^\Theta\| - 2\|\theta_l - c_i^\Theta\| \\
    &\geq \|c_j^\Theta - c_i^\Theta\| - 2\sqrt{\frac{\varepsilon}{1-\varepsilon^2}} \min_{i \neq j}\|c_j^\Theta - c_i^\Theta\| && \mbox{(by definition of $Core(P_i^\Theta)$)} \\
    &\geq \left(1-2\sqrt{\frac{\varepsilon}{1-\varepsilon^2}} \right) \|c_j^\Theta - c_i^\Theta\|
\end{align*}    
    Note the equality is achieved in the last step if $c^\Theta_j$ is the closest center to $c^\Theta_i$.
\end{proof}


\subsection{Proof of Lemma \ref{lem.p.prime.small}}

\pprimesmall*

\begin{proof}
Let $core(P^\Theta_i) = \{l\in P_i^\Theta: \|\theta_l - c_i^\Theta\| \leq \sqrt{\frac{\epsilon}{1-\epsilon^2} }\min_{j\neq i} \|c_i^\Theta - c_j^\Theta\|\}$ as in Lemma \ref{lem.core.set} and consider the points in $core(P^\Theta_i)$ plotted in $M$ space. Suppose towards contradiction that there exists \(l\in core(P_i^\Theta)\) such that $l\in P'_j$ for some $j \neq i$. That is, a point in $core(P^\Theta_i)$ belongs to a different cluster $j$ under $P'$. Then $l$ must satisfy the following condition:
\begin{equation}\label{eq.lemma57.beginning}
  \|m_l - c'_i\| \geq \|m_l - c'_j\|
\end{equation}

We can bound the left hand side from above:
\begin{align*}
    \|m_l - c'_i\| 
    &\leq \|m_l -g(c_i^\Theta) \| + \|g(c_i^\Theta) - c_i'\| && \mbox{(Triangle inequality)}\\
    &\leq \|m_l - g(c_i^\Theta) \| + Lr_i(\Theta). && \mbox{(Lemma \ref{lem.p.prime.1})}
\end{align*}

To bound the right hand side from below, we start from the distance between $m_l$ and $g(c^\Theta_j)$
\begin{align*}
    && \|m_l - g(c_j^\Theta)\| 
    &\leq \|m_l - c_j'\| + \|c_j' - g(c_j^\Theta) \| && \mbox{(Triangle inequality)}\\
    && &\leq \|m_l - c_j'\| + Lr_j(\Theta). && \mbox{(Lemma \ref{lem.p.prime.1})}\\
    \Longrightarrow && \|m_l - g(c_j^\Theta)\|  -  Lr_j(\Theta)&\leq \|m_l - c_j'\| 
\end{align*}

Combining both these bounds with inequality in Equation \eqref{eq.lemma57.beginning}, we have the following bound:
\begin{align*}
   && \|m_l - g(c^\Theta_i)\| + Lr_i(\Theta) 
    &\geq \|m_l - g(c^\Theta_j)\| - Lr_j(\Theta)\\
    \Longleftrightarrow && \|g(\theta_l) - g(c^\Theta_i)\| + Lr_i(\Theta) 
    &\geq \|g(\theta_l) - g(c^\Theta_j)\| - Lr_j(\Theta)
    && \mbox{(by } m_l = g(\theta_l) \mbox{)}\\  
    \Longleftrightarrow && L\|\theta_l - c^\Theta_i\| + Lr_i(\Theta) 
    &\geq \frac{1}{L}\|\theta_l - c^\Theta_j\| - Lr_j(\Theta)
    && \mbox{($L$-bilipschitzness of $g$)}\\
    \Longleftrightarrow && Lr_i(\Theta) + Lr_j(\Theta) 
    &\geq \frac{1}{L}\|\theta_l - c_j^\Theta\| - L\|\theta_l - c_i^\Theta\| &&  \\
   \Longleftrightarrow && r_i(\Theta) + r_j(\Theta) 
    &\geq \frac{1}{L^2}\|\theta_l - c_j^\Theta\| - \|\theta_l - c_i^\Theta\| \\
\end{align*}

The left hand side can be bounded below by Lemma \ref{lem.ri.bound}:
\[
r_i(\Theta) + r_j(\Theta) \leq \frac{2 \varepsilon}{\sqrt{1-\varepsilon^2}}\|c^\Theta_i-c^\Theta_j\|.
\]
The right hand side can be rearranged and bounded below by Lemma \ref{lem.p.prime.2}:
\begin{align*}    
    \frac{1}{L^2}\|\theta_l - c_j^\Theta\| - \|\theta_l - c_i^\Theta\| &= \left(\frac{1}{L^2} - 1 \right)\|\theta_l - c_j^\Theta\| + \|\theta_l - c_j^\Theta\| - \|\theta_l - c_i^\Theta\| \\
     &\geq \left(\frac{1}{L^2} - 1 \right)\|\theta_l - c_j^\Theta\| + \Big(1-2\sqrt{\frac{\varepsilon}{1-\varepsilon^2}}\Big)\|c_j^\Theta -c_i^\Theta\|.
\end{align*}

Combining the upper and the lower bounds and rearranging terms,
\begin{align}
    && \frac{2\varepsilon}{\sqrt{1-\varepsilon^2}}\|c_i^\Theta-c_j^\Theta\| 
    &\geq \left(\frac{1}{L^2} - 1 \right) \|\theta_l - c_j^\Theta\| + \Big(1-2\sqrt{\frac{\varepsilon}{1-\varepsilon^2}}\Big)\|c_j^\Theta -c_i^\Theta\| \nonumber \\
    \Longleftrightarrow && \left(1 -\frac{1}{L^2} \right) \|\theta_l - c_j^\Theta\| &\geq  \Big(1-2\sqrt{\frac{\varepsilon}{1-\varepsilon^2}} - \frac{2\varepsilon}{\sqrt{1-\varepsilon^2}}\Big)\|c_i^\Theta -c_j^\Theta\| \nonumber \\
    \Longleftrightarrow && \left(1 -\frac{1}{L^2} \right) \|\theta_l - c_j^\Theta\| &\geq  \Big( 1 - \frac{ 2 \sqrt{\varepsilon} + 2\varepsilon}{\sqrt{1-\varepsilon^2}} \Big)\|c_i^\Theta -c_j^\Theta\|.  \label{eq.proof.8}
\end{align}

By triangle inequality and since $l$ is in the set $core(P^\Theta_i)$,
\[
\|\theta_l - c_j^\Theta\| \leq
\|\theta_l - c_i^\Theta\| + \|c_i^\Theta - c_j^\Theta\| \leq \Big(1 + \sqrt{\frac{\varepsilon}{1-\varepsilon^2}} \Big)\|c_i^\Theta - c_j^\Theta\|.
\]

We can plug in this bound to Equation \eqref{eq.proof.8} to replace $\|\theta_l - c_j^\Theta\|$ by a function of $\|c_i^\Theta - c_j^\Theta\|$:
\[\left(1 -\frac{1}{L^2} \right)\Big(1 +\sqrt{\frac{\varepsilon}{1-\varepsilon^2}} \Big)\|c_i^\Theta - c_j^\Theta\| \geq \Big(1 - \frac{2\varepsilon + 2\sqrt{\varepsilon}}{\sqrt{1-\varepsilon^2}}\Big) \|c_i^\Theta - c_j^\Theta\|.\]

Dividing both sides by \(\|c_i^\Theta - c_j^\Theta\|\) yields
\begin{align*}
   && \left(1 -\frac{1}{L^2} \right)\Big(1 + \sqrt{\frac{\varepsilon}{1-\varepsilon^2}} \Big) &\geq \Big(1 - \frac{2\varepsilon + 2\sqrt{\varepsilon}}{\sqrt{1-\varepsilon^2}}\Big) \\
  \Longrightarrow && 1 -\frac{1}{L^2}  &\geq \frac{1 - \frac{2\varepsilon + 2\sqrt{\varepsilon}}{\sqrt{1-\varepsilon^2}}}{1 + \frac{\sqrt{\varepsilon}}{\sqrt{1-\varepsilon^2}}} \\
   \Longrightarrow &&  \frac{1}{L^2}  &\leq 1-\frac{1 - \frac{2\varepsilon + 2\sqrt{\varepsilon}}{\sqrt{1-\varepsilon^2}}}{1 + \frac{\sqrt{\varepsilon}}{\sqrt{1-\varepsilon^2}}} = 1 - \frac{\sqrt{1-\varepsilon^2} - 2\varepsilon - 2\sqrt{\varepsilon}}{\sqrt{1-\varepsilon^2} + \sqrt{\varepsilon}} = \frac{2\varepsilon + 3 \sqrt{\varepsilon}}{\sqrt{1-\varepsilon^2} + \sqrt{\varepsilon}} \\
    \Longrightarrow && L  &\geq 
    \left( \frac{\sqrt{1-\varepsilon^2}+\sqrt{\varepsilon}}{2\varepsilon+3\sqrt{\varepsilon}} \right)^{1/2}.  \\
\end{align*}

Since we have assumed that $L < \left( \frac{\sqrt{1-\varepsilon^2}+\sqrt{\varepsilon}}{2\varepsilon+3\sqrt{\varepsilon}} \right)^{1/2}$, this contradicts our assumption. Hence, there is no point $l$ such that  $\theta_l \in core(P^\Theta_i)$ and $m_l \notin P_i'$.

Recall that we are bounding the difference between the partition $P'$ and $P^\Theta$. We showed that if a point belongs to the core set $core(P_i^\Theta)$, then it must belong to $P'_i$ as well. Since the core set contains all but $\varepsilon$ fraction of points, we have $|P_i' \ominus P^\Theta_i| \leq \varepsilon |P^\Theta_i|$. Thus we can finally conclude that $\sum_{i=1}^k|P_i' \ominus P^\Theta_i| \leq 2\varepsilon n$. The final inequality come from the fact that there are at most $n$ points dispersed among the $k$ clusters, so $\sum_{i=1}^k | P^\Theta_i| = n$. Additionally, the symmetric difference counts an error both in the cluster a point came from under $P'$ and the cluster it was added to under $P^{\Theta}$, resulting in the additional factor of 2.
\end{proof}


\subsection{Proof of Lemma \ref{lem.cluster_1}}

\clustconsist*

\begin{proof}
First, we use the second part of Lemma \ref{lem.kmeanscostbound}: $\Delta^2_k(M; P^\Theta) \leq L^4\varepsilon^2 \Delta^2_{k-1}(M)$.
We define the center of the partition $P^\Theta$ in $M$ space as $\{c'_i\}_{i=1}^k = \{ \frac{1}{|P^\Theta_i|} \sum_{l\in P^\Theta_i} m_l \}_{i=1}^k$. 
Note that this partition may no longer be a Voronoi partition in $M$ space. 
Recall that $P'$ is the Voronoi partition induced by centers $\{c'_i\}_{i=1}^k$. Then,
\begin{align*}
    \Delta_k^2(M;P') &= \Delta_k^2(M;P', \{c'_i\}_{i=1}^k) \\
    &\leq \Delta_k^2(M;P^\Theta, \{c'_i\}_{i=1}^k)\\
    &= \Delta_k^2(M;P^\Theta)\\
    &\leq L^4\varepsilon^2 \Delta^2_{k-1}(M)
\end{align*}
The first step is by definition, the second step is because $P'$ is the induced Voronoi partition of $\{c'_i\}_{i=1}^k$, the third step is because $\{c'_i\}_{i=1}^k$ are included cluster means of $P^\Theta$, and the last step is by the second part of Lemma \ref{lem.kmeanscostbound}.
Taking the first, third, and last parts of these inequalities, we obtain,
\[
\Delta^2_k(M; P') \leq \Delta^2_k(M; P^\Theta) \leq L^4\varepsilon^2 \Delta^2_{k-1}(M).
\]

If $P^\Theta$ represented in $M$ space is still a Voronoi partition, then $P'=P^\Theta$. If not, the difference between $P'$ and $P^\Theta$ is bounded by $\sum_{i=1}^k|P_i' \ominus P^\Theta_i| \leq 2\varepsilon n$ by Lemma \ref{lem.p.prime.small}.
We will first bound the difference between $P'$ and $P^M$, and then use this fact to conclude that $P^\Theta \approx P^M$.

Lemma \ref{lem.kmeanscostbound} gives the $L^2\varepsilon$-separation of $M$ and that  $\Delta^2_k(M; P') \leq L^4\varepsilon^2 \Delta^2_{k-1}(M)$.
With these two conditions, we can now instantiate Theorem \ref{thm.ostrovsky} from \citet{ostrovsky2013effectiveness}, which shows that
for a well-separated space (such as $M$), if there is a partition that yields sufficiently small $k$-means cost (such as $P'$), then there exists another nearly identical partition with distinct centers.
We instantiate the theorem with $\alpha=L^4\varepsilon^2$.
To satisfy the theorem's condition on $\alpha$, we impose $\varepsilon \leq \frac{1}{\sqrt{801}L^2}$, so that $\alpha = L^4\varepsilon^2 \leq \frac{1-401L^4 \varepsilon^2}{400}$.
Then, Theorem \ref{thm.ostrovsky} tells us that for each cluster $P'_i$, we can match it with one of the optimal clusters $P^M_{\sigma(i)}$, with only a small error:
\[
|P'_i \ominus P^M_{\sigma(i)}| \leq 161L^4\varepsilon^2|P^M_{\sigma(i)}|,
\]
where $\sigma(i)$ is some bijection. By summing this over all $i \in [k]$ and using the assumption $L^2\varepsilon < \frac{1}{\sqrt{801}}$, we get
\begin{equation}\label{eq.pprimeandpm}
\sum_{i=1}^k|P'_i \ominus P^M_{\sigma(i)}| \leq \sum_{i=1}^k  161L^4\varepsilon^2|P^M_{\sigma(i)}| \leq 161L^4\varepsilon^2n \leq 6L^2\varepsilon n ,
\end{equation}
where $n = |\Theta| = \sum_{i=1}^k|P^M_{\sigma(i)}|$. 

Combining this with the bound of Lemma \ref{lem.p.prime.small} showing that  $\sum_{i=1}^k|P_i' \ominus P^\Theta_i| \leq 2\varepsilon n$, gives the final desired bound:
\[
\sum_{i=1}^k|P^\Theta_i \ominus P^M_{\sigma(i)}| \leq \sum_{i=1}^k|P^\Theta_i \ominus P'_{i}| + \sum_{i=1}^k|P'_i \ominus P^M_{\sigma(i)}| \leq 2\varepsilon n + 6L^2\varepsilon n \leq 8L^2\varepsilon n = O(L^2\varepsilon n).
\]
The first step is the triangle inequality, the second step applies Lemma \ref{lem.p.prime.small} and Equation \eqref{eq.pprimeandpm}, and the remainder combines and simplifies terms.

The application of Lemma \ref{lem.p.prime.small} requires $\varepsilon<0.1$ and $L^2 \leq \frac{\sqrt{1-\varepsilon^2}+\sqrt{\varepsilon}}{2\varepsilon+3\sqrt{\varepsilon}}$.
Combining this with the assumption required to apply Theorem \ref{thm.ostrovsky} that
$L^2\varepsilon \leq \frac{1}{\sqrt{801}}$,
yields the requirement that 
$L^2 \leq \min(\frac{1}{\sqrt{801}\varepsilon}, \frac{\sqrt{1-\varepsilon^2}+\sqrt{\varepsilon}}{2\varepsilon+3\sqrt{\varepsilon}} )$.
Note that for $\varepsilon>0.011$, it always holds that $L^2 \leq \frac{1}{\sqrt{801}\varepsilon} \leq \frac{\sqrt{1-\varepsilon^2}+\sqrt{\varepsilon}}{2\varepsilon+3\sqrt{\varepsilon}}$, and for smaller $\varepsilon$ the condition on $L$ becomes relaxed, making it easier to satisfy.
\end{proof}


\subsection{Proof of Lemma \ref{lem.eta.bound}}

\etabound*

\begin{proof}
Observe that for any row of a matrix $A$,
$\|A_i\| = \|U_i \Sigma V^\top \| = \|U_i \Sigma \|$,
where $U, \Sigma, V$ describe the singular value decomposition of $A$.
Then,
$ \max_i \|A_i\| \leq \sigma_1 \sqrt{\sum_{j=1}^r U_{i,j}^2} = \sigma_1 = ||A||$.
Using this, we bound $\eta$ from above as follows.
\begin{align}
\eta &= \max_i \| m_i - \tilde{m}_i \|\\
&\leq \| M - \tilde{M} \| \notag\\
&= \| M - (M+E) + (M+E) - \tilde{M} \| \notag\\
&\leq \| M - (M+E)\| + \|(M+E) - \tilde{M} \| \notag\\
&= \| E\| + \|(M+E) - \tilde{M} \|,\label{eq.intermediate}
\end{align}
where the first inequality is by the fact shown above (that $ \max_i \|A_i\| \leq ||A||$), 
and the second inequality is from the triangle inequality.

Next we bound the second term in Equation \eqref{eq.intermediate}.
Let the true rank of $M$ be $r$. Then by Eckart-Young Theorem, we have
\begin{align*}
\|(M+E) - \tilde{M} \|
&= \min_{A: \; rank(A) = r}\|(M+E) - A \|  && \mbox{(Eckart-Young Theorem)}\\
&\leq  \|(M+E) - M\|\\
&= \|E\|.
\end{align*}
Applying this to Equation \eqref{eq.intermediate} to bound $\eta$, we obtain $\eta \leq 2\|E\|$. 

Next we use this upper bound to derive a high probability bound on $\eta$.
By Gordon's theorem (Theorem \ref{Gaussian}), $\E[\|E\|] \leq s(\sqrt{n} + \sqrt{T})$. Thus,
\[
    \E[\eta] \leq \E[2\|E\|] \leq 2s(\sqrt{n} + \sqrt{T}).
\]

Instantiating Markov's inequality on $\eta$, we have that with probability at least $1-\delta$,
\begin{align*}
\eta \leq \frac{2s(\sqrt{n} + \sqrt{T})}{\delta}.
\end{align*}
\end{proof}


\subsection{Proof of Lemma \ref{lem.noisy.intermediate1}}

\noisyintone*

\begin{proof}

First we establish the relationship between $\Delta^2_k(\tilde M; \tilde P)$ and $\Delta^2_{k}(M)$ as follows.
\begin{align}
  \Delta^2_k(\tilde M; \tilde P) \leq
  \Delta^2_k(\tilde M; P^M, \{c_i^{M}\}_{i=1}^k)
  &= \sum_{i=1}^k \sum_{l\in  P_i^M} \|\tilde m_l - c_i^M\|^2 \nonumber\\
   &= \sum_{i=1}^k \sum_{l\in  P_i^M} \|\tilde m_l - m_l + m_l - c_i^M\|^2 \nonumber\\
  &= \sum_{i=1}^k \sum_{l\in  P_i^M} \| m_l - c_i^M\|^2 + \|\tilde{m}_l - m_l\|^2 + 2\langle \tilde{m}_l - m_l , m_{l} - c_i^M \rangle \nonumber\\
   &\leq \sum_{i=1}^k \sum_{l\in  P_i^M} \| m_l - c_i^M\|^2 + \|\tilde{m}_l - m_l\|^2 + 2\| \tilde m_l - m_l\| \cdot \|m_{l} - c_i^M \| \nonumber\\
   &= \Delta_k^2(M) + \sum_{i=1}^k \sum_{l\in  P_i^M}  \underbrace{\|\tilde{m}_l - m_l\|^2}_{\leq \eta^2} + 2\underbrace{\| \tilde m_l - m_l\|}_{\leq \eta} \cdot \underbrace{\|m_{l} - c_i^M \|}_{\leq 2\sqrt{T}} \nonumber\\
   &\leq \Delta_k^2(M) + \sum_{i=1}^k \sum_{l\in P_i^M} \left( \eta^2 + 4\eta \sqrt{T} \right) \nonumber\\
   &= \Delta_k^2(M) + n\eta^2 + 4n\eta \sqrt{T} \label{eq.tilde.first}
\end{align}
The first inequality is by Cauchy-Schwarz and the second is from the worst case bound
$\|m_{l} - c_i^M \| \leq \max_{i \neq j} \|m_{i} - m_j \| 
 \leq \| \mathbf{1}_T - (-\mathbf{1}_T) \| = 2\sqrt{T} $, where $\mathbf{1}_T = (1, 1, 1, \ldots, 1, 1)$ denotes a vector of length $T$ with all elements being $1$, and because there are $n$ points indexed by $l$ across the $k$ clusters.

Next, we compare the bound for $k-1$ clusters.
Let $P^{\tilde{M},k-1}$ be the optimal $(k-1)$-means partition for $\tilde M$ and let $\{c_i^{\tilde{M},k-1}\}_{i=1}^{k-1}$ be the corresponding centers. Then,
\begin{align}
    \Delta_{k-1}^2(M) &\leq 
    \Delta_{k-1}^2(M; P^{\tilde{M},k-1}) \nonumber\\
    &= \sum_{i\in [k-1]} \sum_{l\in P^{\tilde{M},k-1}_i} \|m_l - c^{\tilde M,k-1}_i\|^2 \nonumber\\
    &= \sum_{i\in [k-1]} \sum_{l\in P^{\tilde{M},k-1}_i} \|m_l - \tilde m_l + \tilde m_l - c^{\tilde{M},k-1}_i\|^2 \nonumber\\
    &\leq \sum_{i\in [k]} \sum_{l\in P^{\tilde M,k-1}_i} \|\tilde m_l - c^{\tilde{M},k-1}_i\|^2  + \|m_l - \tilde  m_l\|^2 + 2\|m_l - \tilde  m_l\| \cdot \|\tilde m_l - c^{\tilde{M},k-1}_i\| \nonumber\\
    &\leq \Delta_{k-1}^2(\tilde M) + n\eta^2 + 2 n\eta \max_{i\neq j} \|\tilde m_i - \tilde m_j \| \nonumber\\
    &\leq \Delta_{k-1}^2(\tilde M) + n\eta^2 + 2 n\eta \left( \max_{i\neq j} \|m_i - m_j \| + 2\eta \right) \nonumber\\
    &\leq \Delta_{k-1}^2(\tilde M) + n\eta^2 + 4 n\eta (\sqrt{T}+\eta) \label{eq.tilde.second}
\end{align}

The second inequality is by Cauchy-Schwarz, and the third inequality is because the distance between any point and a center should be smaller than the diameter of points, i.e., the distance between two furthest points: 
$\|\tilde m_l - c^{\tilde{M},k-1}_i\|  \leq \max_{i\neq j} \|\tilde m_i - \tilde m_j \|$.  The fourth inequality is an extended application of the triangle inequality: $\|\tilde m_i-\tilde m_j\| \leq \|\tilde m_i - m_i\|+\|m_i-m_j\|+\|m_j-\tilde m_j\| \leq \|m_i-m_j\| + 2\eta$, where the final step there comes from conditioning on the good event $G$. The last inequality of \eqref{eq.tilde.second} is again because $\max_{i \neq j} \|m_{i} - m_j \| 
 \leq 2\sqrt{T}$.
 
Combining inequalities in \eqref{eq.tilde.first} and \eqref{eq.tilde.second} gives the following bound:
\begin{align*}
\Delta^2_k(\tilde M; \tilde P)
   &\leq \Delta_k^2(M) + n\eta^2 + 4n\eta \sqrt{T} && \mbox{(Inequality \eqref{eq.tilde.first})}\\
   &\leq \varepsilon^2 L^4\Delta_{k-1}^2(M) + n\eta^2 + 4n\eta \sqrt{T} && \mbox{(Lemma \ref{lem.kmeanscostbound})}\\
   &\leq \varepsilon^2 L^4\Delta_{k-1}^2(\tilde M) + \varepsilon^2 L^4(n\eta^2 + 4 n\eta (\sqrt{T}+\eta) )+ n\eta^2 + 4n\eta \sqrt{T} 
   && \mbox{(Inequality \eqref{eq.tilde.second})}\\
   &= \varepsilon^2 L^4\Delta_{k-1}^2(\tilde M)+ \varepsilon^2 L^4( 5n\eta^2 + 4 n\eta \sqrt{T} )+ n\eta^2 + 4n\eta \sqrt{T} \\
   &= \varepsilon^2 L^4\Delta_{k-1}^2(\tilde M)+ n\eta \left( ( 5\varepsilon^2 L^4+1 )\eta + (4\varepsilon^2 L^4 + 4) \sqrt{T} \right) \\
   &\leq \varepsilon^2 L^4 \Delta^2_{k-1}(\tilde M)+ n\eta \left(2\frac{1}{4}\eta +5\sqrt{T} \right)
   &&\mbox{($\varepsilon L^2 < \frac{1}{2}$)}\\
  &\leq \varepsilon^2 L^4 \Delta^2_{k-1}(\tilde M)+ n\eta (3\eta + 5\sqrt{T}).
\end{align*}

The next step is to convert this upper bound to by $2\varepsilon^2 L^4 \Delta^2_{k-1}(\tilde M)$, so then we can bound the $k$-means cost $\Delta_k^2(\tilde M ; \tilde P)\leq \alpha \Delta_{k-1}^2(\tilde M)$ for $\alpha=2\varepsilon^2L^4$. To do so, we want to show that $2n\eta (3\eta + 4\sqrt{T}) \leq \varepsilon^2 L^4 \Delta^2_{k-1}(\tilde M)$.
By Lemma \ref{lem.eta.bound}, conditioned on the good event $G$, the following bound holds:
\begin{align*}
    n\eta (3\eta + 5\sqrt{T})
    &\leq \frac{2ns(\sqrt{n} + \sqrt{T})}{\delta} \left( \frac{6s(\sqrt{n} + \sqrt{T})}{\delta} + 5\sqrt{T} \right).
\end{align*}
Then the goal is to find a condition on $s$ such that
\begin{align*}
     \frac{2ns(\sqrt{n} + \sqrt{T})}{\delta} \left( \frac{6s(\sqrt{n} + \sqrt{T})}{\delta} + 5\sqrt{T} \right)
    & \leq \varepsilon^2 L^4 \Delta^2_{k-1}(\tilde M).
\end{align*}
By rearranging terms to get a quadratic form in $s$,
\begin{align*}
   && s \left( \frac{6s(\sqrt{n} + \sqrt{T})}{\delta} + 5\sqrt{T} \right)
    & \leq \frac{\delta \varepsilon^2 L^4 \Delta^2_{k-1}(\tilde M)}{2n(\sqrt{n} + \sqrt{T})}\\
    \Longleftrightarrow && \frac{6(\sqrt{n} + \sqrt{T})}{\delta}s^2 + 5\sqrt{T}s
   -\frac{\delta \varepsilon^2 L^4 \Delta^2_{k-1}(\tilde M)}{2n(\sqrt{n} + \sqrt{T})} &\leq 0.
\end{align*}

Define this quadratic formula as $f(s) := \frac{6(\sqrt{n} + \sqrt{T})}{\delta}s^2 + 5\sqrt{T}s
   -\frac{\delta \varepsilon^2 L^4 \Delta^2_{k-1}(\tilde M)}{2n(\sqrt{n} + \sqrt{T})}$.
Since $\frac{6(\sqrt{n} + \sqrt{T})}{\delta}>0$ and $f(s=0)<0$, we can find the condition on $s$ that makes $f(s)\leq0$ by finding the positive solution for $f(s)=0$.
Using the quadratic formula, we get
\begin{align*}
    s &= \frac{-5\sqrt{T} \pm \sqrt{25T + \frac{24(\sqrt{n}+\sqrt{T})}{\delta} \frac{\delta \varepsilon^2 L^4 \Delta^2_{k-1}(\tilde M)}{2n(\sqrt{n} + \sqrt{T})} } } {12(\sqrt{n}+\sqrt{T})/\delta}\\
    &=  \frac{ \delta \left(-5\sqrt{T} \pm \sqrt{25T + \frac{ 12\varepsilon^2 L^4 \Delta^2_{k-1}(\tilde M)}{n} } \right)} {12(\sqrt{n}+\sqrt{T})}\\    
    &=  \frac{ \delta \left(-\sqrt{25T} \pm \sqrt{25T + \frac{ 12\varepsilon^2 L^4 \Delta^2_{k-1}(\tilde M)}{n} } \right)} {12(\sqrt{n}+\sqrt{T})} .\\
\end{align*}

Therefore, $2n\eta (3\eta + 4\sqrt{T})\leq \varepsilon^2 L^4 \Delta^2_{k-1}(\tilde M)$ is satisfied for all 
\[
s \in \left( 0, \frac{ \delta \left(-\sqrt{25T} + \sqrt{25T + \frac{ 12\varepsilon^2 L^4 \Delta^2_{k-1}(\tilde M)}{n} } \right)} {12(\sqrt{n}+\sqrt{T})} \right).
\]

Finally, using the fact that $\tilde M$ is \(4L^2\varepsilon\)-separated in the event of $G$, we can use Theorem \ref{thm.ostrovsky} with $\alpha = 2 \varepsilon^2 L^4$ to conclude
that \(|\tilde P \ominus P^{\tilde M}| \leq 161 \cdot 16 L^4 \varepsilon^2 n = 2576 L^4 \varepsilon^2 n\). 

\end{proof}


\subsection{Proof of Lemma \ref{lem.noisy.intermediate2}}

\noisyinttwo*

\begin{proof}
Define
\[
core(P_i^M) = \Big\{l\in P_i^M:  \|m_l - c_i^M\| \leq \sqrt{\frac{L^2\varepsilon}{1-L^4\varepsilon^2}}\min_{i\neq j} \|c_i^M - c_j^M\|\Big\}.
\]

Similar to Lemma \ref{lem.core.set},
we construct a uniform probability distribution over $\forall l\in P_i^M$:
 \[\mathbb{E}[\|m_l - c_i^M\|^2] = \frac{1}{|P_i^M|} \sum_{l\in P_i^M} \|m_l - c_i^M\|^2 = r_i^2(M)\]
Usking Markov's inequality with the fact that
$M$ is $L^2\varepsilon$-separated, we see that,
 \[\mathbb{P}(\|m_l - c_i^M\|^2 \geq t) \leq \frac{r_i^2(M)}{t} 
 \leq \frac{\frac{L^4\epsilon^2}{1-L^4\epsilon^2} \min_{j\neq i} \|c_i^M - c_j^M\|^2}{t}.
 \]
Take $t = \frac{r_i^2(M)}{L^2\varepsilon}= \frac{L^2\varepsilon}{1-L^4\epsilon^2} \min_{j\neq i} \|c_i^M - c_j^M\|^2 $.
Then,
 \[
 \mathbb{P}\left(\|m_l - c_i^M\| \geq \sqrt{t}\right) 
 =  \mathbb{P}\left(m_i \notin core(P_i^M) \right) 
 \leq L^2\varepsilon.
 \]
Hence, there are at most a $L^2\varepsilon$ fraction of points $l$ in $P^M_i$ that will not belong to $core(P^M_i)$.

Similar to Lemma \ref{lem.p.prime.2},
for all $l \in core(P_i^M)$ and for all $j\neq i$,
\begin{equation}\label{eq.to.contradict}
    \|m_l - c_j^M\| - \|m_l - c_i^M\| \geq \Big(1-2\sqrt{\frac{L^2\varepsilon}{1-L^4\varepsilon^2}}\Big)\min_{i\neq j} \|c_i^M - c_j^M\|.
\end{equation}

We will assume towards a contradiction that there exists a $j$ such that,
\[
\|\tilde m_l - c_j^M\|  \leq \|\tilde m_l - c_i^M\|,
\]
and show a contradiction to conclude that all $l \in core(P^M_i)$ should also belong to $\tilde P_i$.
From this assumption, we apply triangle inequality twice to yield:
\[
\|m_l - c_j^M\| - \|\tilde m_l - m_l\|  \leq \|m_l - c_i^M\| + \|\tilde m_l - m_l\|.
\]
Rearranging gives,
\[
\|m_l - c_j^M\| - \|m_l - c_i^M\| \leq 2\|\tilde m_l - m_l\| \leq 2\eta.
\]
Conditioning on the good event $G$, we can apply the bound on $\eta$ from Lemma \ref{lem.eta.bound}:
\[
\|m_l - c_j^M\| - \|m_l - c_i^M\| \leq\frac{4s(\sqrt{n} + \sqrt{T})}{\delta}.
\]
With the assumption that $s < \frac{\delta \Big(1-2\sqrt{\frac{L^2\varepsilon}{1-L^4\varepsilon^2}}\Big)\min_{i\neq j} \|c_i^M - c_j^M\|}{4(\sqrt{n} + \sqrt{T})}$, 
this bound becomes
\[
\|m_l - c_j^M\| - \|m_l - c_i^M\| <  \Big(1-2\sqrt{\frac{L^2\varepsilon}{1-L^4\varepsilon^2}}\Big)\min_{i\neq j} \|c_i^M - c_j^M\|,
\]
which contradicts the inequality in \eqref{eq.to.contradict}.
Hence we conclude \( core(P_i^M) \subseteq \tilde P_i\). 

Recall that \(|core(P_i^M)| \geq (1-L^2\varepsilon)|P_i^M|\).
Hence, all but $L^2\varepsilon$ fraction of the points will not change the cluster assignment, and for those who change, the symmetric set difference will count the error twice. 
Therefore, \(|\tilde P \ominus P^M| = \sum_{i=1}^k |\tilde P_i \ominus P^M_i| \leq 2L^2\varepsilon n\).  
\end{proof}


\subsection{Proof of Lemma \ref{lem.cluster_2}}

\cluster*

\begin{proof}

To instantiate Lemmas \ref{lem.noisy.intermediate1} and \ref{lem.noisy.intermediate2}, we need to assume
\[
s < \min\Biggl\{
\frac{ \delta \left(-\sqrt{25T} + \sqrt{25T + \frac{ 12\varepsilon^2 L^4 \Delta^2_{k-1}(\tilde M)}{n} } \right)} {12(\sqrt{n}+\sqrt{T})}, 
\frac{\delta \Big(1-2\sqrt{\frac{L^2\varepsilon}{1-L^4\varepsilon^2}}\Big)\min_{i\neq j} \|c_i^M - c_j^M\|}{4(\sqrt{n} + \sqrt{T})} 
\Biggl\}.
\]
Since $\tilde M$ is $4L^2\varepsilon$-separated conditioned on $G$, then also,
\[
\Delta^2_{k}(\tilde M) \leq 16\varepsilon^2 L^4 \Delta^2_{k-1}(\tilde M).
\]
Based on this, we use a slightly stronger bound on $s$:
\[
s < \min\Biggl\{
\frac{ \delta \left(-\sqrt{25T} + \sqrt{25T + \frac{ 3\Delta^2_{k}(\tilde M)}{4n} } \right)} {12(\sqrt{n}+\sqrt{T})}, 
\frac{\delta \Big(1-2\sqrt{\frac{L^2\varepsilon}{1-L^4\varepsilon^2}}\Big)\min_{i\neq j} \|c_i^M - c_j^M\|}{4(\sqrt{n} + \sqrt{T})} 
\Biggl\}.
\]
Define $s_1 := \frac{ \delta \left(-\sqrt{25T} + \sqrt{25T + \frac{ 3\Delta^2_{k}(\tilde M)}{4n} } \right)} {12(\sqrt{n}+\sqrt{T})}$ and $s_2 := \frac{\delta \Big(1-2\sqrt{\frac{L^2\varepsilon}{1-L^4\varepsilon^2}}\Big)\min_{i\neq j} \|c_i^M - c_j^M\|}{4(\sqrt{n} + \sqrt{T})} $.
We will simplify this bound by showing that $s_1$ (the left element in minimum) is asymptotically smaller than $s_2$ (the right element).

First, we show an upper bound of $s_1$.
Note that $\sqrt{x+y}-\sqrt{x} \leq \frac{y} {2\sqrt{x}}, \; \forall x, y>0$ because the second derivative is always negative. Hence,
\begin{align*}
s_1 &= \frac{ \delta \left(-\sqrt{25T} + \sqrt{25T + \frac{ 3\Delta^2_{k}(\tilde M)}{4n} } \right)} {12(\sqrt{n}+\sqrt{T})}\\ 
&\leq \frac{\delta} {\sqrt{n}+\sqrt{T}} \cdot \frac{\Delta^2_{k}(\tilde M)} {160n\sqrt{T}} && (\sqrt{x+y}-\sqrt{x} \leq \frac{y} {2\sqrt{x}}, \; \forall x, y>0)\\
&\leq \frac{\delta} {\sqrt{n}+\sqrt{T}} \cdot \frac{2n\sqrt{T}} {160n\sqrt{T}} && (\Delta^2_{k}(\tilde M)\leq 2n\sqrt{T})\\
&=\frac{\delta} {\sqrt{n}+\sqrt{T}} \cdot \frac{1} {80}. 
\end{align*}

Next, we show a lower bound for $s_2$, the second element in the minimum.
\begin{align*}
    s_2 & = \frac{\delta \Big(1-2\sqrt{\frac{L^2\varepsilon}{1-L^4\varepsilon^2}}\Big)\min_{i\neq j} \|c_i^M - c_j^M\|}{4(\sqrt{n} + \sqrt{T})}\\
    &\geq \frac{\delta(1-\frac{\sqrt{89}}{267}) \min_{i\neq j} \|c_i^M - c_j^M\| }{4(\sqrt{n} + \sqrt{T})} && (L^2\varepsilon<1/\sqrt{801})\\
    &\geq \frac{\delta(1-2\frac{\sqrt{89}}{267}) \sqrt{\frac{1-\varepsilon^2}{\varepsilon^2} \min_ir_i^2(M) } }{4(\sqrt{n} + \sqrt{T})} && (\mbox{Lemma \ref{lem.ri.bound}})\\
    &\geq \frac{\delta(1-2\frac{\sqrt{89}}{267}) \sqrt{99} \min_ir_i(M) }{4(\sqrt{n} + \sqrt{T})} && (\varepsilon<0.1)\\
    &\geq \frac{2.311 \cdot \delta \min_ir_i(M) }{(\sqrt{n} + \sqrt{T})} \\
    &\geq \frac{\delta }{\sqrt{n} + \sqrt{T}} \;2 \cdot \min_ir_i(M).
\end{align*}
With our assumption $\min_ir_i(M)\geq1/160$, we can combine the bounds on $s_1$ and $s_2$ to get:
\[
s_1 \leq \frac{\delta} {\sqrt{n}+\sqrt{T}} \cdot \frac{1} {80}  \leq \frac{\delta }{\sqrt{n} + \sqrt{T}} \;2 \cdot \min_ir_i(M) \leq s_2
\]
Hence, by assuming $s<s_1 \leq \frac{\delta} {\sqrt{n}+\sqrt{T}} \frac{\Delta^2_{k}(\tilde M)} {160n\sqrt{T}} = O(\frac{\delta \sqrt{T}}{\sqrt{n}+\sqrt{T}})$ (recall $\Delta_k^2(\tilde M) = O(nT)$), we satisfy the constraint for $s$ and can apply  Lemmas \ref{lem.noisy.intermediate1} and \ref{lem.noisy.intermediate2}.

Instantiating these lemmas yields the desired guarantee:
\begin{align*}
|P^M\ominus P^{\tilde M}| 
&\leq |P^M\ominus \tilde P| + |\tilde P \ominus P^{\tilde M}| && \mbox{(triangle inequality)}\\
&\leq 2L^2\epsilon n + 161 \cdot 16 L^4\epsilon^2n && \mbox{(Lemmas \ref{lem.noisy.intermediate1} and \ref{lem.noisy.intermediate2})}\\
&\leq 2L^2\epsilon n + 92 L^2 \epsilon 
n  && (L^2\varepsilon<1/\sqrt{801})\\
&= 94L^2\epsilon n.
\end{align*}
\end{proof}

\section{Omitted Proofs from Section \ref{sec.theory.improve}}

\subsection{Proof of 
Theorem \ref{thm.pre.bound}}

\thmprebound*

\begin{proof}
Let $n_A = \alpha^2 n$ for some $\alpha \in (0, 1)$. 
From Lemma \ref{lem.pre.mse}, the upper bound on pre-intervention MSE when we use the full donor pool $X$ is,
\begin{align*}
    \MSE(\hat{m}_0^-;X)\leq \frac{\mu^2}{T_0}\mathbb{E}[ ( \sigma^*_X + 2s(\sqrt{n} + \sqrt{T})) ^2 ] + \frac{2s^2r}{T_0}.
\end{align*}
When we adopt ClusterSC, the full donor pool $X$ is substituted by the selected donor set $A$, and several terms will change in this upper bound: $n$ is replaced by $n_A$, $r$ by $r_S$, and $\sigma^*_X$ by $\sigma^*_A$. These changes only decrease the bound, and we provide further analysis on the improvement in this upper bound by showing the dependence on $n$ (the number of units) and $s^2$ (noise).
Specifically, we focus on the terms inside the expectation in this bound since $\frac{\mu^2}{T_0}$ does not change and $\frac{2s^2r}{T_0}$ can only decrease.

By expanding the terms inside the expectation, we get
\begin{align*}
    \left(\sigma^*_X + 2s(\sqrt{n} + \sqrt{T})\right) ^2
    &= \sigma_X^{*2} + 4s (\sqrt{n} + \sqrt{T}) \sigma^*_X + 4s^2(\sqrt{n} + \sqrt{T})^2 \\
    &= \sigma_X^{*2} + 4s (\sqrt{n} + \sqrt{T}) \sigma^*_X 
    + 4s^2n + 8s^2\sqrt{nT} + 4s^2T.
\end{align*}
When we change the donor matrix to $A$ instead of $X$, this changes to
\begin{align*}
    \left(\sigma^*_A + 2s(\alpha \sqrt{n} + \sqrt{T})\right) ^2
    &= \sigma_A^{*2} + 4s (\alpha \sqrt{n} + \sqrt{T}) \sigma^*_A 
    + 4s^2\alpha^2n + 8s^2\alpha \sqrt{nT} + 4s^2T.
\end{align*}
Then, the difference between the two becomes
\begin{align}
    \label{eq.pre.proof.1}
    &\left(\sigma^*_X + 2s(\sqrt{n} + \sqrt{T}) \right)^2 - \left(\sigma^*_A + 2s(\alpha \sqrt{n} + \sqrt{T})\right) ^2\\
    &= \sigma_X^{*2} - \sigma_A^{*2} + 4s(\sqrt{n}+\sqrt{T}) (\sigma^*_X - \sigma^*_A ) - (1-\alpha)4s\sqrt{n}\sigma^*_A
    + (1-\alpha^2)4s^2n + (1-\alpha)8s^2\sqrt{nT}\notag\\
    &= (\sigma^*_X + \sigma^*_A + 4s(\sqrt{n}+\sqrt{T})) (\sigma^*_X - \sigma^*_A ) - (1-\alpha)4s\sqrt{n}\sigma^*_A
    + (1-\alpha^2)4s^2n + (1-\alpha)8s^2\sqrt{nT} \notag
\end{align}
Let $\Delta$ denote the quantity in Equation \eqref{eq.pre.proof.1}, i.e., $\Delta \coloneqq \left(\sigma^*_X + 2s(\sqrt{n} + \sqrt{T}) \right)^2 - \left(\sigma^*_A + 2s(\alpha \sqrt{n} + \sqrt{T})\right) ^2$. To lower bound the expectation of $\Delta$, we apply Theorem \ref{thm.singval.gauss}, which also requires the assumptions that $r<T$ (which is satisfied automatically in our model) and $n_A < n + 4T -4 \sqrt{nT}$: 
\begin{align*}
    \E[\Delta] \geq & \;
    (2 \sigma_A^* + s(5-\alpha)\sqrt{n} + 2s\sqrt{T} )((1-\alpha)s\sqrt{n} - 2s\sqrt{T}) \\
    &- 4(1-\alpha)s\sqrt{n}\sigma^*_A
    + 4(1-\alpha^2)s^2n + 8(1-\alpha)s^2\sqrt{nT}\\
    =& \; 2(1-\alpha)s\sqrt{n}\sigma^*_A -4(1-\alpha)s\sqrt{n}\sigma^*_A -4s\sqrt{T}\sigma^*_A
    + s^2(5-\alpha)(1-\alpha)n + 2s^2(1-\alpha)\sqrt{nT}\\
    &-2s^2(5-\alpha)\sqrt{nT} -4s^2T + (1-\alpha^2)4s^2n + (1-\alpha)8s^2\sqrt{nT}\\
    =& \; -2(1-\alpha)s\sqrt{n}\sigma^*_A -4s\sqrt{T}\sigma^*_A
    + s^2( (9-6\alpha-3\alpha^2)n -4T -8\alpha^2\sqrt{nT})\\
    =& \; -s\left(2 (1-\alpha) \sqrt{n} +4\sqrt{T} \right)\sigma^*_A
    + \underbrace{3(\alpha+3)(1-\alpha)s^2n}_{\text{dominating term}} - \left(4s^2T +8s^2\alpha^2\sqrt{nT}\right).
\end{align*}
The first and the third terms are negative but they are relatively small numbers compared to the middle one (highlighted as dominating term). Hence, for sufficiently large $n$, $\E[\Delta]= \Omega(s^2n)$.

Finally, The difference in the two upper bounds is
\begin{align*}
    & \frac{\mu^2}{T_0} \E[\left(\sigma^*_X + 2s(\sqrt{n} + \sqrt{T})\right)^2] + \frac{2s^2r}{T_0}
    - \frac{\mu^2}{T_0} \E[\left(\sigma^*_A + 2s(\alpha\sqrt{n} + \sqrt{T})\right) ^2] - \frac{2s^2r_S}{T_0}\\
    &=  \frac{\mu^2}{T_0} \E[\left(\sigma^*_X + 2s(\sqrt{n} + \sqrt{T})\right)^2 - \left(\sigma^*_A + 2s(\alpha\sqrt{n} + \sqrt{T})\right)^2] + \frac{2s^2}{T_0}(r-r_S)\\
    &= \Omega(s^2n),
\end{align*}
since $n \gg T>T_0$, $\mu$ is a constant, and $r-r_S \geq 0$.
\end{proof}

\subsection{Proof of 
Theorem \ref{thm.post.bound}}

\postbound*

\begin{proof}
Let $n_A = \alpha^2 n$ for some $\alpha \in (0, 1)$. For this result we want to investigate the gap between post-intervention RMSE upper bounds presented in Lemma \ref{lemma.post.rmse} under the full donor pool $X$ and the sleected donor pool $A$ in ClusterSC.
This upper bound under $X$ is:
\[
RMSE(\hat m_0^+ ; X) \leq
\frac{\eta}{\sqrt{T-T_0}}\mathbb{E}[\sigma^*_X + 2s(\sqrt{n} + \sqrt{T}) ] + \sqrt{n}(\mu + \eta),
\]
We then compare the difference between this bound under $X$ and the same expression under $A$, where $A$ is the selected donor pool by ClusterSC:
\begin{align*}
& \quad \frac{\eta}{\sqrt{T-T_0}}\mathbb{E}[\sigma^*_X -\sigma^*_A + 2s(1-\alpha)\sqrt{n}] + (1-\alpha)\sqrt{n}(\mu + \eta)\\
&\geq \frac{\eta}{\sqrt{T-T_0}}\left(s((1-\alpha)\sqrt{n}-2\sqrt{T}) + 2s(1-\alpha)\sqrt{n}\right) + (1-\alpha)\sqrt{n}(\mu + \eta)\\
&= \frac{\eta}{\sqrt{T-T_0}}\left(s (3(1-\alpha)\sqrt{n}-2\sqrt{T})\right) + (1-\alpha)\sqrt{n}(\mu + \eta),
\end{align*}
where the first inequality comes from Theorem \ref{thm.singval.gauss}. To instantiate this theorem requires the assumptions that $r<T$ (which is satisfied automatically in our model) and $n_A < n + 4T -4 \sqrt{nT}$. 
Since $n\gg T$ and $\mu$ and $\eta$ are constants, we conclude that the difference is $\Omega(s\sqrt{n})$.
\end{proof}


\section{Additional Experiments with Synthetic Datasets}\label{app.simulation}

In this section, we share more detailed results with synthetic simulated datasets.
Appendix \ref{olsridge} shows the performance of ClusterSC with OLS and Ridge regression (via Robust Synthetic Control).
Appendix \ref{sec.lasso} presents an analysis of ClusterSC with Lasso regression.

\subsection{ClusterSC with Robust Synthetic Control (OLS and Ridge regression)}
\label{olsridge}
Robust synthetic control \citep{rsc} first applies de-noising step (HSVT) and then learns weights using OLS or ridge regression. This is simply adopting OLS or ridge regression in step 3 of Algorithm \ref{alg.sc.core}.
In this section, we show the results comparing the performance of robust synthetic control against our ClusterSC. The number of distinct signals in each submatrices $A$ and $B$ is $r_A=r_B=3$, and the ridge coefficient was fixed to $0.01$.

Figure \ref{fig.ridge_ols_box} shows the average MSE per dataset over varying noise levels ($s$), using 1) robust synthetic control with OLS (blue), 2) robust synthetic control with ridge (orange), 3) ClusterSC with OLS (green), and 4) ClusterSC with ridge (red).
We observe that ridge regression performs better than OLS with or without the clustering step.
With ClusterSC, we reduce the expectation of MSE and the variance as well, regardless of the choice of regression method (OLS or ridge).

\begin{figure}[thb]
\centering
\includegraphics[width=0.9\textwidth]{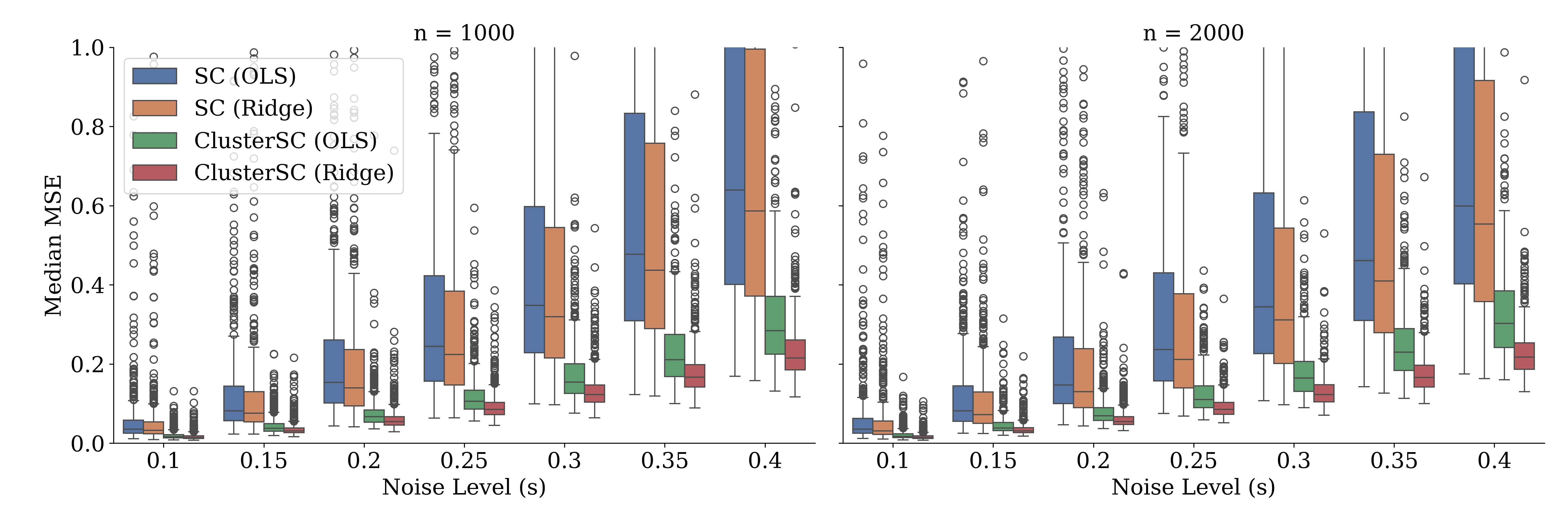}
\caption{Median post-intervention MSE, measured per dataset. Each boxplot corresponds to ridge, OLS, cluster and then ridge, and cluster and then OLS, from left to right, plotted for each noise level. Left plot is with $n=1000$ donor units in total and the right plot is with $n=2000$.
}\label{fig.ridge_ols_box}
\end{figure}

Next, we define the pairwise improvement for a target $i$ as the difference in post-intervention MSE scores between the two methods: $I_i= \MSE(\hat{m}_i^+;X)-\MSE(\hat{m}_i^+;A)$.
Then, we take $\text{median}(I_i)$ as a metric to assess the overall improvement measured from $n$ SC instances constructed from one dataset (leave-one-out placebo test).
Figure \ref{fig.ridge_ols_line} shows the pairwise improvement (i.e., $\text{median}(I_i)$) induced by ClusterSC at varying noise level. 
We observe that the median improvement is almost always positive, meaning that more than half of the individuals benefit from using ClusterSC instead of RSC. 
In the $n=2000$ case, the improvement grows as noise increases, corroborating our Theorem \ref{thm.post.bound}.
On the other hand, when $n=1000$, the improvement continues to increase until $s=0.35$, after which it plateaus at $s=0.4$ for Ridge and decreases for OLS. This may be attributed to the noise level $s=0.4$ being sufficiently high that the donor $n=1000$ is not enough to effectively capture the true signal. Nonetheless, a significant improvement is observed in median MSE for the same setting from Figure \ref{fig.ridge_ols_box}.
The improvement is more stable (low variance) with higher $n$, and the improvement in OLS and ridge regression is not too different.

\begin{figure}[h]
\centering
\includegraphics[width=0.9\textwidth]{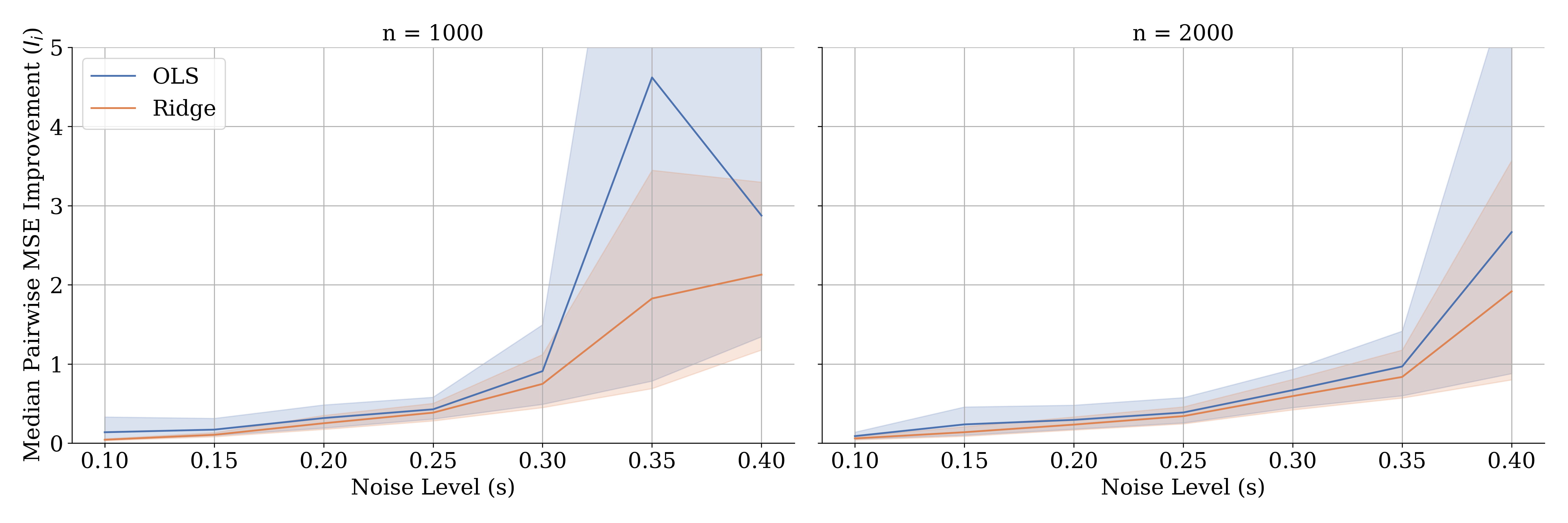}
\caption{Median of the pair-wise improvements measured by comparing SC and ClusterSC, when using OLS (blue) and Ridge regression (orange), over different noise levels.
}\label{fig.ridge_ols_line}
\end{figure}

\subsection{ClusterSC with Lasso regression}\label{sec.lasso}

In this section, we use Lasso regression for step 3 of Algorithm \ref{alg.sc.core}.
Again, we use the same data generating method with the same parameters $n_A=n_B\in\{500, 1000\}$, $T=10$, and $r_A=r_B=3$. The Lasso coefficient was $0.01$ for all experiments.
Due to the high computational cost of Lasso, we only test for noise levels $s\in\{0.1, 0.2, 0.3, 0.4\}$.

Figure \ref{fig.hsvt_lasso} shows the average post-intervention MSE. We observe more improvement with clustering as noise level increases. Compared to the results in Figure \ref{fig.ridge_ols_box}, the improvement induced by the clustering step when regression is performed with Lasso is not as large as compared to OLS or Ridge in absolute value. Still, the improvement is evident.

\begin{figure}[h]
\centering
\includegraphics[width=0.9\textwidth]{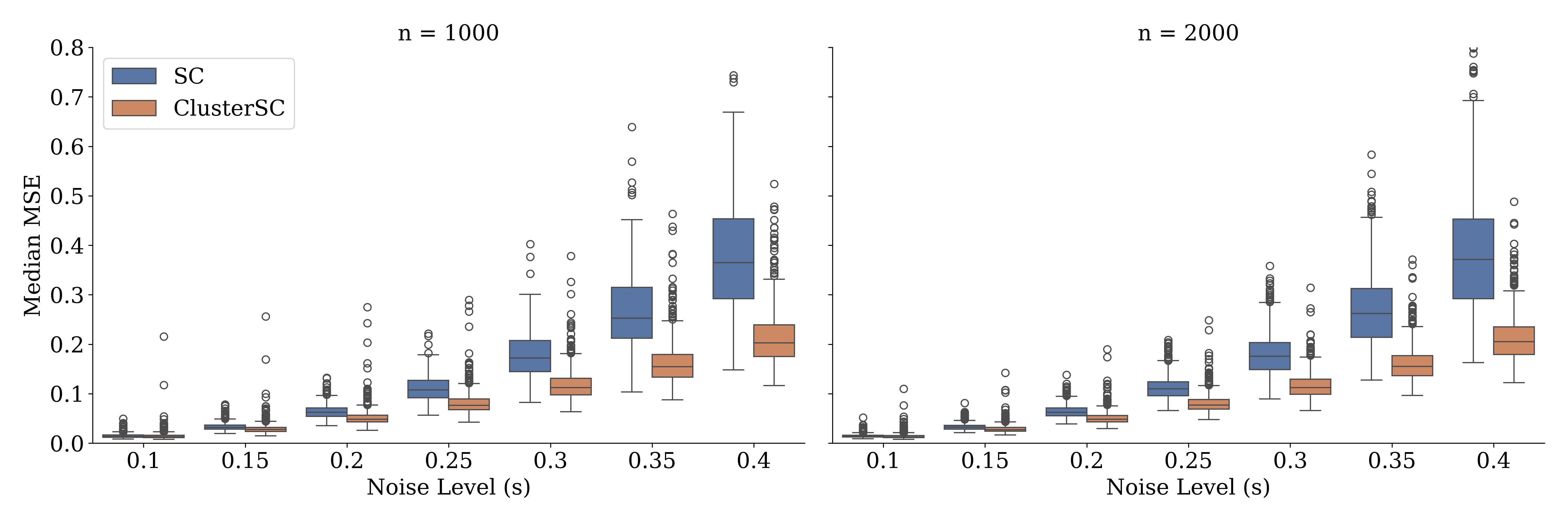}
\caption{Median post-intervention MSE, measured per dataset. We compare using SC with Lasso (blue) and ClusterSC with Lasso (orange).}\label{fig.hsvt_lasso}
\end{figure}

\textbf{Analysis on Active Donors.}
To further investigate, we analyze the \emph{active donors} selected by Lasso regression, which correspond to donor units with non-zero SC weights. Since Lasso produces a sparse vector, it effectively selects a subset of relevant donors to reconstruct the target unit. In our experimental setup, units in group $A$ share the same signals, and all target units are sampled from $A$. Ideally, the relevant donors should be chosen from $A$ rather than $B$.
To quantify this, we use precision scores to assess the proportion of selected donors that correctly belong to group $A$.

\begin{figure}[t]
\centering
\includegraphics[width=0.95\textwidth]{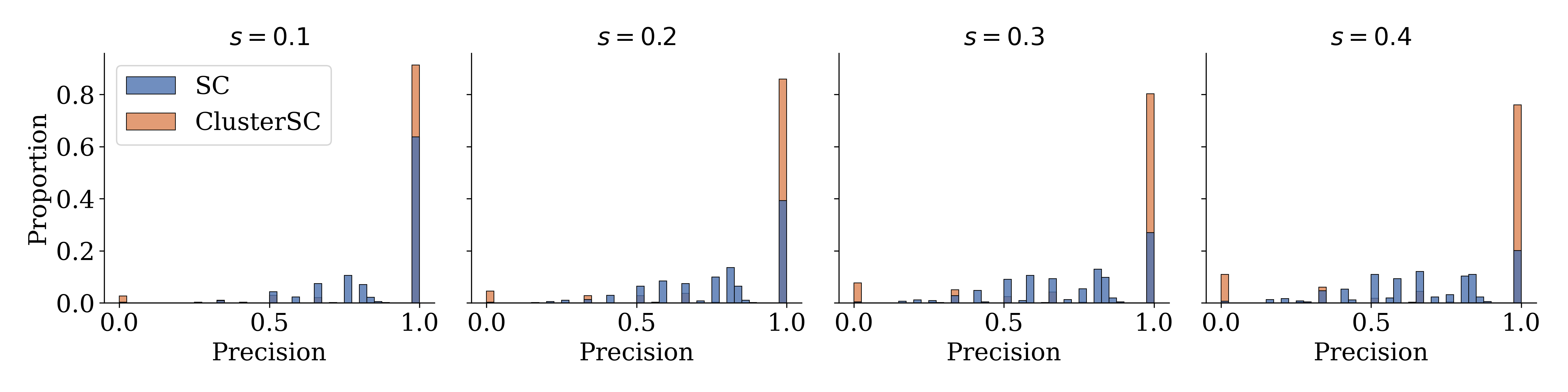}
\caption{Histogram of the precision score of active donor units from using SC with Lasso (blue) and ClusterSC with Lasso (orange). 100 iterations are displayed for $n=1000$, each with 150 leave-one-out scores.
}\label{fig.lasso_precision}
\end{figure}

Figure \ref{fig.lasso_precision} illustrates the distribution of precision scores for active donor units selected by SC with Lasso (blue) and ClusterSC with Lasso (orange).
For all noise regimes, ClusterSC shows more concentrated precision score around $1$ compared to SC.
As noise increases, this difference in concentration increases further, indicating that ClusterSC is more robust to noise in selecting relevant donors.

For SC, a small proportion of donors from group $B$ are incorrectly included, with their precision scores spread relatively evenly across all score ranges.
When ClusterSC fails to achieve high precision, the scores tend to be near $0$, rather than being somewhere in between $0$ and $1$.
This pattern suggests that ClusterSC selected an incorrect cluster (e.g., the selected donor set predominantly consists of group B), and hence it can only choose the donors from B.
Nonetheless, incorrect selection of donor set does not occur frequently, which agrees with ClusterSC's decreased median MSE compared to SC (see Figure \ref{fig.hsvt_lasso}).

\begin{figure}[h]
\centering
\includegraphics[width=0.95\textwidth]{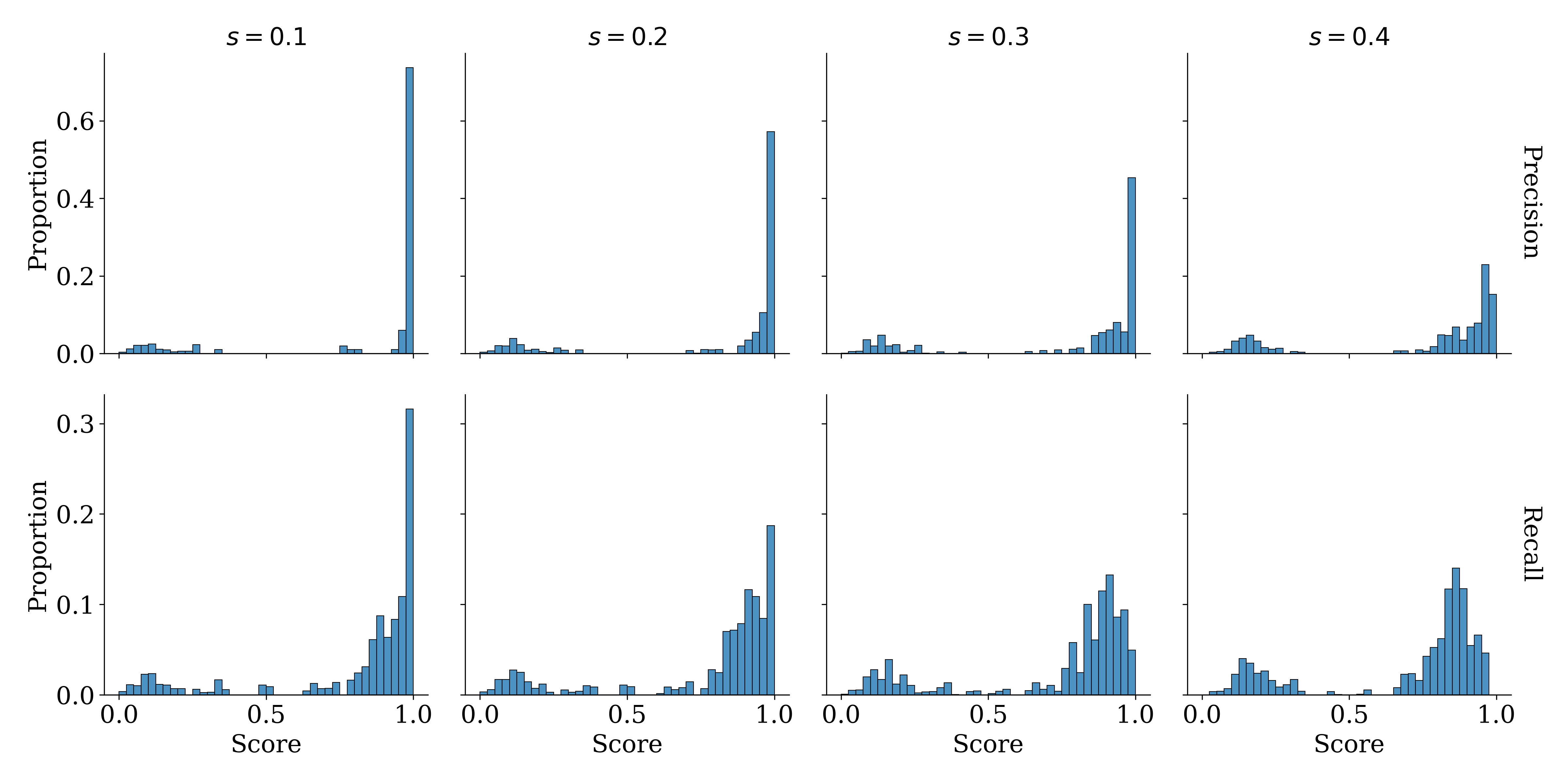}
\caption{Histograms of precision and recall scores of donor units selected by ClusterSC compared to units in $A$, over varying noise levels. 100 iterations are displayed for $n=1000$, each with 150 leave-one-out scores.
}\label{fig.cluster_precision_recall}
\end{figure}

\textbf{Analysis on Clusters.}
We can further analyze this improvement by investigating how effectively the clustering step (k-means) selects a relevant donor set by computing its precision and recall scores with units in group $A$ (from Step 4 of Algorithm \ref{alg.csc}) as the true label. Figure \ref{fig.cluster_precision_recall} presents histograms of precision scores (top row) and recall scores (bottom row) for varying noise levels. 
A precision score close to $1$ indicates that most of the donors selected by ClusterSC are already from the relevant group $A$, making it easier for Lasso to select the best fit among them.
In the low noise regimes, the precision score is $1$ for more than $70\%$ of the cases, and the regression step does not need to filter it any further.
However, the ability of ClusterSC to select only the relevant donors (high precision score) degrades as noise increases. ClusterSC, together with the power of Lasso to learn sparse weights, can significantly improve the precision scores in high noise regimes, from the first row of Figure \ref{fig.cluster_precision_recall} to the orange plots in Figure \ref{fig.lasso_precision}.
In contrast, without clustering, Lasso must filter out irrelevant donors from group $B$ solely through the power of regularization in the regression step, which is shown in blue bars in Figure \ref{fig.lasso_precision}.
This shows that ClusterSC with Lasso has a synergistic effect for only selecting relevant donors, improving from using only Lasso (Figure \ref{fig.lasso_precision}) or clustering (Figure \ref{fig.cluster_precision_recall}) individually.

\textbf{Computational Efficiency.}
Another advantage that ClusterSC brings in with Lasso regression is computational efficiency.
Like Lasso, ClusterSC seeks to isolate only the most important donors for target reconstruction. Unlike Lasso, ClusterSC avoids running a linear regression in the full $n$ dimensions.
ClusterSC comprises of three main parts: 1) SVD on the original matrix, 2) clustering, and 3) regression (on the subsampled donor).
We recall the computational complexity of each of these steps.

\begin{lemma}[\citep{ostrovsky2013effectiveness}]\label{lem:kmeans_runtime}
Fix any \(\omega > 0\) and a dataset \(X\in \mathbb{R}^{n\times d}\). Assuming \(\Delta_k^2(X) \leq \varepsilon^2 \Delta_{k-1}^2(X)\) for \(\varepsilon\) small enough, there is an algorithm which, with constant probability, outputs a partition \(\hat P\) that is \((1+\omega)\)-optimal solution to \(k\)-means on \(X\), meaning \(\Delta_k^2(X;\hat P) \leq (1+\omega)\Delta_k^2(X)\). Furthermore, this algorithm runs in time \(O(2^{O(k(1+\varepsilon^2)/\omega)} nd)\). 
\end{lemma}

\begin{lemma}[\cite{golub2013matrix}]\label{lem:svd_runtime}
Computing the singular value decomposition for a dense \(m\times n\) matrix takes time \(O(mn\min\{m,n\})\).
\end{lemma}

\begin{lemma}[\cite{efron2004least}]\label{lem:lasso_runtime}
    Lasso regression on \(v\) variables (number of features) with sample size \(s\) (number of observations) each takes time \(O(v^3 + v^2s)\).
\end{lemma}

Note that in synthetic control, the number of features for regression purposes is actually the number of donors \(n\), as the goal is to predict the behavior of the target donor per-time-step. Suppose we use a constant-factor $k$-means approximation for ClusterSC (i.e., the algorithm will correctly identify clusters for all but $\omega$ fraction of points with $\omega = \Omega(1)$).
We compare the runtime of synthetic control with Lasso versus ClusterSC with Lasso in terms of $n$ (number of all donor units), $n_A$ (number of donors selected by ClusterSC), and $m$ (the number of targets we test).
Note the we do not consider $T$ as we assume a tall matrix ($n \gg T$).

For ClusterSC (Algorithm \ref{alg.csc}), the major computations will be:
\begin{itemize}
    \item Step 1. Learn clusters: $O(n)+O(n)$ \hfill(Lemmas \ref{lem:svd_runtime} and \ref{lem:kmeans_runtime})
    \item Step 3. Construct donor matrix $A$ and denoise: $O(n_A)$ \hfill(Lemma \ref{lem:svd_runtime})
    \item Step 4. ($m$ rounds of) SC Learning: $m\cdot O(n_A^3)$ \hfill(Lemma \ref{lem:lasso_runtime})
\end{itemize}
Note that we only $n$ and $n_A$ to grow (where $n>n_A$), but not $T$. Hence, considering runtime in terms of parameters $n, n_A$, and $m$, the time complexity of ClusterSC is $O(n+ mn_A^3)$.
On the other hand, the classical synthetic control with Lasso will have time complexity of $O(mn^3)$.

\section{Additional Details about the Housing Dataset}\label{app.housing}

We provide more insights into the housing price index dataset used in Section \ref{sec.realworld}.
Figure \ref{fig.hpi} provides a graphical summary of this time series panel dataset.

\begin{figure}[h]
\centering
\includegraphics[width=0.7\linewidth]{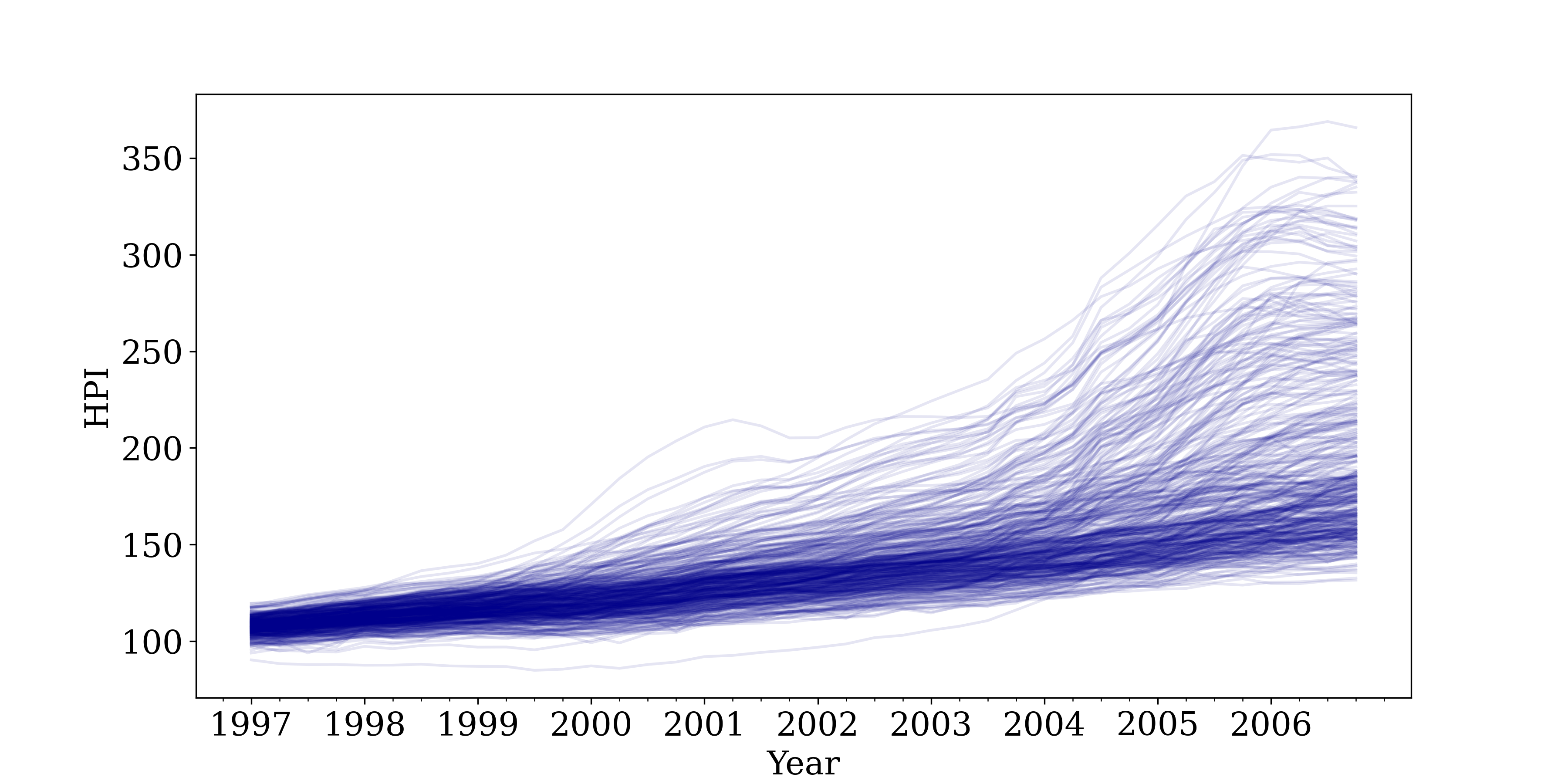}
\caption{Full time series plot of cleaned housing price index (HPI) dataset with $n=400$ metropolitan areas and $T=40$ quarters from $1997$ to $2006$.}
\label{fig.hpi}
\end{figure}

To further understand the importance of the singular value cutoff, we plot the singular value spectrum of the entire dataset in Figure \ref{fig.hpi_singval}.
Note that in every iteration with different train test split, we recompute SVD for the donor matrix, so the spectrum may not be exactly the same across iterations. However, this does still give us an understanding of the dataset, which we assume to be approximately low rank.
On the left side, we plot the cumulative singular value ratio, which shows that the top three or four singular values contain about $95\%$ of the total singular values.
On the right side, we can see the gap in singular values decreases as we increase the index, which shows that the dataset does satisfy the assumption of approximately low-rank structure.

\begin{figure}[h]
\centering
\includegraphics[width=0.8\textwidth]{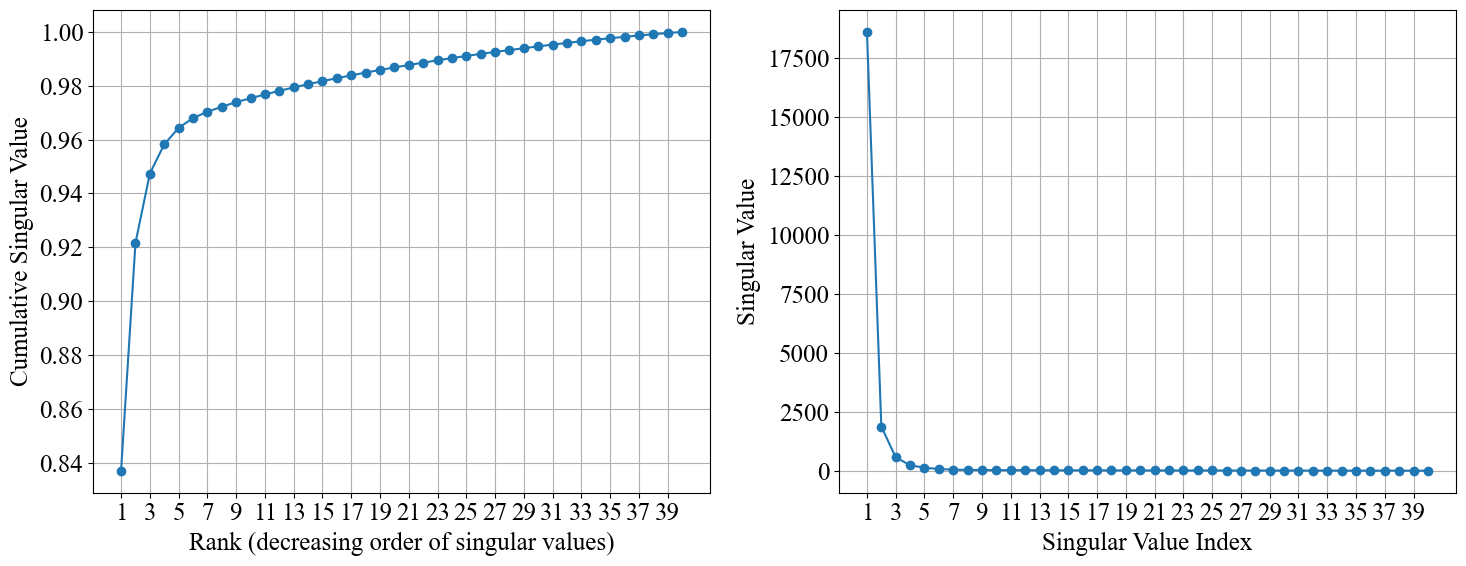}
\caption{Cumulative singular value ratio (left) and the singular value spectrum (right), ordered by decreasing singular values.}\label{fig.hpi_singval}
\end{figure}